%% file: neurips_2026.tex
\documentclass{article}

\usepackage[main, final, nonatbib]{neurips_2026}

\usepackage[utf8]{inputenc} 
\usepackage[T1]{fontenc}    
\usepackage{hyperref}       
\usepackage{url}            
\usepackage{booktabs}       
\usepackage{amsfonts}       
\usepackage{nicefrac}       
\usepackage{microtype}      
\usepackage{xcolor}         

\usepackage[style=alphabetic, natbib=true]{biblatex}
\usepackage{amsmath}
\usepackage{amssymb}
\usepackage{graphicx}
\usepackage[dvipsnames]{xcolor}
\usepackage{hyperref}
\usepackage{cleveref}
\usepackage{mathtools}
\usepackage{adjustbox}
\usepackage{subcaption}
\usepackage{color-edits}
\addauthor{kb}{cyan}
\addauthor{pasha}{red}
\addauthor[Mahdi]{m}{blue}
\addauthor[Danny]{dm}{Green}
\usepackage{amsthm}
\usepackage{csquotes}
\usepackage{booktabs}
\usepackage{multirow}
\usepackage{algorithm}
\usepackage{algorithmic}

\newcommand{\R}{\mathbb{R}}

\newcommand{\Event}{\mathrm{Event}}

\newtheorem{theorem}{Theorem}[section]

\newtheorem{lemma}[theorem]{Lemma}
\newtheorem{corollary}[theorem]{Corollary}

\newtheorem{invariant}{Invariant}

\crefname{invariant}{invariant}{invariants}
\Crefname{invariant}{Invariant}{Invariants}
\usepackage{thmtools}

\AddToHook{cmd/appendix/before}{\crefalias{section}{appendix} \crefalias{subsection}{appendix}}

\title{Efficient Dynamic Algorithms for Graph Neural Networks with Non-Linear Propagation}

\author{%
  Kiarash Banihashem\thanks{Equal contribution.} \\
  University of Maryland, College Park \\
  \texttt{kiarash@umd.edu} \\
  \And
  MohammadTaghi Hajiaghayi\footnotemark[1] \\
  University of Maryland, College Park \\
  \texttt{hajiagha@umd.edu} \\
  \And
  Mahdi JafariRaviz\footnotemark[1] \\
  University of Maryland, College Park \\
  \texttt{mahdij@umd.edu} \\
  \And
  Silvio Lattanzi\footnotemark[1] \\
  Google Research \\
  \texttt{silviol@google.com} \\
  \And
  Danny Mittal\footnotemark[1] \\
  University of Maryland, College Park \\
  \texttt{mittal@umd.edu} \\
}

\begin{document}

\maketitle

\begin{abstract}
Graph Neural Networks (GNNs) are widely used for representation learning on graphs, but most methods assume static topologies, making them inefficient on evolving networks where edges change over time. Existing dynamic approaches either model graph evolution through temporal GNN architectures without focusing on efficient dynamic maintenance, or are restricted to linear propagation models based on Personalized PageRank.

In this work, we study how to efficiently maintain node representations for non-linear GNN propagation under edge insertions and deletions. The propagation has no learned parameters, and only a classifier applied afterward is trained. For a broad class of standard activation functions, we develop a residual-based dynamic algorithm that selectively propagates local errors via push operations, maintaining an approximation to the evolving fixed point without full recomputation. We prove that our method achieves amortized $O(1/\epsilon)$ update time per graph change under a degree-normalized error guarantee. Our approach uses a potential-based analysis in a degree-scaled norm and, in contrast to prior work on the linear case, requires no randomness assumptions on either the update sequence or the input vector. For the linear special case, we additionally provide an exact dynamic algorithm via low-rank matrix inverse updates. Experiments on benchmark datasets show that incorporating non-linearity improves accuracy while preserving efficient update performance, yielding a scalable and theoretically grounded method for maintaining this propagation on dynamic graphs.
\end{abstract}

\input{sections/1_intro.tex}

\input{sections/2_related_work.tex}

\input{sections/3_dyn_nonlin.tex}

\input{sections/4_exact.tex}

\input{sections/5_experiments.tex}

\input{sections/6_conclusion.tex}

\clearpage
\printbibliography


\appendix
\input{sections/appendix.tex}


\end{document}

%% file: sections/1_intro.tex
\section{Introduction}
Graph Neural Networks (GNNs) are widely used in machine learning, both for prediction tasks such as node classification and as general-purpose tools for representation learning~\cite{hamilton2017inductive,kipf2017semi}. 
Their ability to exploit the underlying graph structure of real-world data makes them effective in many domains such as biology, social sciences, and transportation~\cite{wu2020comprehensive}. 
A key challenge, however, is that real-world data is often \emph{dynamic}: the relationships between the nodes of the graph are constantly changing, requiring the model to update its outputs accordingly.
A naive way of satisfying this requirement is to periodically recompute the model after each change in order to incorporate changes in the graph. However, such recomputation is computationally expensive and therefore undesirable both theoretically and practically.

Existing approaches to dynamic GNNs can be broadly divided into two lines of work. 
The first line is \emph{temporal} GNNs which
model graph evolution using recurrent architectures, attention mechanisms, or memory modules~\cite{kumar2018learning,trivedi2019dyrep,Xu2020Inductive,rossi2020temporal}. 
These approaches are primarily geared toward prediction tasks such as node classification and are not designed for efficiently maintaining representations under graph changes~\cite{zheng2022instant,zheng2023decoupled}.
The second line focuses on the propagation component of GNNs and considers \emph{linear} GNNs where each node has a value obtained via iterated neighborhood averaging~\cite{zheng2022instant,zheng2023decoupled}. In this setting, the model reduces to \emph{Personalized PageRank (PPR)}, where 
the structure enables efficient updates, allowing the representation vectors to be maintained under graph changes without recomputing from scratch~\cite{BahmaniCG10, zhang2016approximate}.

A natural question is whether one can efficiently maintain the representations of \emph{non-linear} GNNs under graph updates. 
Non-linearity is central to the success of GNNs, as it enables the model to capture complex patterns in the data. 
By repeatedly composing non-linear activations with neighborhood aggregation, GNNs gain expressive power analogous to deep neural networks in standard (non-graph) settings.


\paragraph{Our Contributions.} In this paper, we answer the above question affirmatively. 
We study GNNs in which each node applies a non-linear activation to a weighted average of its neighbors. 
The propagation has no learned parameters, and its fixed point is unique. Training happens only in a classifier applied to the propagated features, as in InstantGNN~\cite{zheng2022instant}.
We assume the activation function is $K$-Lipschitz for some $K \le 1$, a condition that is sufficient for convergence. This restriction is mild in the GNN setting. Common activations such as ReLU, tanh, sigmoid, hard tanh, softplus, softsign, and ELU are 1-Lipschitz or smaller.
Under this assumption, we show that one can maintain an approximate solution for all nodes in amortized time $O(1/\epsilon)$, where $\epsilon$ is the 
degree-normalized
approximation parameter. More precisely, the amortized time is $O\big((1+B_z)/(1-K(1-\alpha))^2 \cdot (1/\epsilon)\big)$, where $\alpha$ is the teleportation probability and $B_z$ bounds the size of the fixed point (\Cref{section:preliminaries}). For $K=1$, as for ReLU and tanh, the factor $1-K(1-\alpha)$ equals $\alpha$. In addition, for the linear special case, we provide an exact dynamic maintenance algorithm based on low-rank matrix-inverse updates, obtaining a worst-case update time $O(n^{1.5275})$, where $n$ is the number of nodes.
We complement our theoretical results with empirical evidence demonstrating the importance of non-linearity for achieving high accuracy. On large dynamic graph benchmarks, we show the practical behavior of our proposed approximation algorithm.

\paragraph{Our Techniques.}
To maintain an approximate solution under dynamic updates, we use a residual vector, as in prior work \citep{zheng2022instant,zhang2016approximate}. We use a push operation to propagate large residuals, and we show that when all residuals are small relative to the degree-scaled error, the maintained solution is a valid approximation. For an edge update, two effects must be handled: (i) the new values that should be propagated through the updated edge, and (ii) the effect of the degree changes at the two endpoints on their neighbors, which arises from the edge-weight normalization. The first effect can be handled in $O(1)$ time using the scaling trick of \citet{zhang2016approximate}, which is not part of the main algorithm of \citet{zheng2022instant}. After applying the effect of the edge update, the residuals at the endpoints may increase, and we propagate them when needed. The main difference from prior work is in the analysis of this cleanup step. Our setting requires handling the nonlinear update rule, and we also remove the randomness assumption used in prior work.  We remove this assumption and obtain a deterministic guarantee, at the cost of a weaker amortized update bound.

%% file: sections/2_related_work.tex
\paragraph {Related Work}
A line of work in dynamic GNNs studies temporal graphs, where the primary focus is on prediction tasks over evolving graphs~\cite{kumar2018learning,trivedi2019dyrep,Xu2020Inductive,rossi2020temporal}.
These approaches are typically divided into \emph{discrete-time dynamic} (DTD) and \emph{continuous-time dynamic} (CTD) models: the former takes as input a sequence of graph snapshots at different time steps, while the latter operates on a stream of timestamped events that modify the graph.
We refer to~\cite{zheng2025survey} for a comprehensive overview.
These models use the timing and order of graph changes as a signal for prediction, and some of them, such as \citet{rossi2020temporal}, keep a memory for each node that depends on past events.
In contrast to this line of work, our focus is not on prediction accuracy but on \emph{efficiently maintaining} node embeddings under graph updates with \emph{theoretical guarantees}.
In particular, we provide amortized update time guarantees for our algorithm.
The embeddings we maintain depend only on the current graph, so our method fits tasks where the history of changes does not matter, while temporal GNNs fit tasks where it does.

The closest line of work to ours is the literature on dynamic \emph{Personalized PageRank} (PPR), which can be viewed as a special case of our model with update function $f_i(x)=x$. 
For this problem, \cite{zhang2016approximate} build on the Push framework of \cite{andersen2006local} and give a dynamic algorithm that maintains the PageRank solution for a randomly chosen source node. In our notation, this corresponds to setting $s$ to a random basis vector. Their algorithm has amortized update time $O(1/(n\epsilon))$ when $\beta \in \{0,1\}$, where $\beta$ is the normalization parameter in \Cref{eq:def_W}. 
Building on this line of work, \cite{zheng2022instant} study a linear GNN model and obtain the same $O(1/(n\epsilon))$ update time for general $\beta$, but under an additional randomness assumption on the edge updates: each edge is assumed to be equally likely to be inserted at each time. In contrast to these works, we allow nonlinear updates and make no randomness assumptions on either the update sequence or the vector $s$.

A separate line of work on dynamic PageRank uses random walk sampling to maintain the solution vector~\cite{BahmaniCG10, bahmani2012pagerank, JayaramLMOS24}. On a high level, this approach estimates each node's solution value by averaging over a collection of random walks starting from the node. When an edge is inserted or deleted, only the affected walks are re-sampled in order to maintain correctness. This line of work was initiated by \cite{BahmaniCG10}, who analyzed the setting under random edge updates. More recently, \cite{JayaramLMOS24} extend this approach to general (adversarial) updates and obtain upper and lower bounds for both additive $\ell_1$ error and multiplicative error. 
The sampling technique heavily relies on the linearity of PageRank, which allows a node's solution to be computed independently of its neighbors' solutions. As such, this approach does not extend to the nonlinear setting considered in our work.

\paragraph{Relation to other GNN models.} As in the GNN of \citet{scarselli2008graph}, our node values are the fixed point of a contraction, but our update function is fixed rather than learned. Our model is closest to InstantGNN~\cite{zheng2022instant}, which also trains only a classifier on the propagated features, and to APPNP~\cite{gasteiger2018predict}, whose propagation also has no learned parameters. Unlike the most popular GNNs, our propagation has no learned weights, no mixing across feature dimensions, and no attention. Maintaining a version with a learned matrix that mixes the features is left open.

%% file: sections/3_dyn_nonlin.tex
\section{Dynamic Non-linearity}

\subsection{Preliminaries} \label{section:preliminaries}
To model nonlinearity, we consider an undirected, unweighted graph $G = (V,E)$ with permanent self-loops at every node. We assume $V = \{1, \dots, n\}$, where $n$ is the number of nodes. \Cref{table:params} in the Appendix lists all the parameters used in the paper and their roles. Let $A$ be the adjacency matrix of this graph. We define the normalized weight matrix $W$ as
\begin{equation}
    W = D^{-\beta}AD^{\beta-1},
    \label{eq:def_W}
\end{equation}
where $0 \le \beta \le 1$ is the normalization parameter, and $D$ is the diagonal degree matrix.
\footnote{We adopt this normalization as it is standard in related literature. Any matrix $W$ with a spectral radius $\rho(W) \le 1$ imposes a similar convergent dynamic with a unique fixed point, though the dynamic algorithm may vary.}
We denote the degree of node $i$ by $d(i)$. Since every node has a self-loop, $d(i) \ge 1$ for all $i$,  $W$ is well-defined.

Given a vector $s = (s_1, \dots, s_n)^T \in \R^n$ and teleportation probability $\alpha$, we consider the following dynamics. Let $z(t) = (z_1(t),\dots,z_n(t))^T$ be the vector of node values at time $t$. We define $z(t+1)$ from $z(t)$ by
\begin{equation}
    z_i(t+1) = f_i\left( \alpha s_i + (1 - \alpha) \sum_j w_{ij} z_j(t) \right),
\end{equation}
for every $i \in V$, where $f_i$ is a node-wise function. In vector form, this update is
\begin{equation}
    z(t+1) = F\left( \alpha s + (1 - \alpha) W z(t) \right),
\end{equation}
where $F \colon \R^n \to \R^n$ applies the functions $f_i$ pointwise. The vector $s$ holds the input values. We also define the induced map
\begin{math}
    T(z) = F\left(\alpha s + (1 - \alpha) W z\right).
\end{math}
When $F$ is the identity map, this recovers the linear model considered by \cite{zheng2022instant}.

Throughout the paper, we use the following assumptions:
\begin{itemize}
    \item $\alpha \in (0,1)$, and $\beta \in [0,1]$;
    \item each $f_i$ is $K$-Lipschitz with $K \le 1$, that is, for all $x,y \in \R$, $|f_i(x) - f_i(y)| \le K |x-y|$;
    \item there is a constant $B_z$ such that $\|D^{\beta-1} z^*\|_\infty \le B_z$ for the fixed point $z^*$ of $T$ on any graph.
    \begin{itemize}
        \item If each $f_i$ takes values in $[-B_z,B_z]$, the condition holds (\Cref{lem:zs-bounded-01} in the Appendix).
        \item For any activation, it holds with $B_z = (\max_i |f_i(0)| + K\alpha\|s\|_\infty)/(1-K(1-\alpha))$ (\Cref{lem:zs-bounded-general} in the Appendix). For the identity and ReLU, this gives $B_z = \|s\|_\infty$.
    \end{itemize}
\end{itemize}

We will use the degree-scaled infinity norm, defined as $\|v\|_{D,\infty} = \|D^{\beta-1} v\|_\infty$. The degree-scaled infinity norm satisfies the property in the following lemma, which we use later. Using this norm and the assumptions above, we then show that the fixed point $z^*$ is unique by the Banach fixed-point theorem \citep{banach1922operations}.

\begin{restatable}{lemma}{lemDegScaledNormW} \label{lem:deg-scaled-norm-W}
For any vector $v$, $\|Wv\|_{D,\infty} \le \|v\|_{D,\infty}$.
\end{restatable}

\begin{restatable}{lemma}{lemUniqueFixedPoint} \label{lem:unique-fixed-point}
Assume that $\alpha \in (0,1)$, $\beta \in [0,1]$. Further assume that each $f_i$ is $K$-Lipschitz with $K \le 1$. Then there is a unique fixed point $z^*$ satisfying $T(z^*) = z^*$. Moreover, the iteration $z(t+1) = T(z(t))$ converges to $z^*$ for any initial point $z(0)$ as $t \to \infty$.
\end{restatable}

The proofs of these lemmas are provided in \Cref{appendix:proofs:nl}.

\subsection{The State}
Our goal is to maintain an approximation of the unique fixed point $z^*$ in an evolving graph. The idea is to keep an estimate $z$ of $z^*$, together with a record of how far $z$ is from being a fixed point. To do this, we maintain three vectors $z,y,r$, which satisfy the following invariants:
\begin{invariant} \label{inv:1}
    $y = \alpha s + (1 - \alpha)Wz$.
\end{invariant}
\begin{invariant} \label{inv:2}
    $z + r = F(y)$.
\end{invariant}

Under these invariants, $z$ is our approximation of $z^*$, and $r$ is the residual vector containing changes that have not yet been propagated. The vector $y$ stores the averaged neighbor values after teleportation. It can be recomputed from $z$ and the current graph, but we maintain it explicitly to support fast graph updates later. If node $i$ were updated once, its value would change from $z_i$ to $f_i(y_i) = z_i + r_i$. So $r_i$ measures how much node $i$ would still change, and $r = 0$ exactly when $z = z^*$.

The next lemma shows that a small residual implies that $z$ is close to $z^*$.

\begin{lemma} \label{lem:r-small-to-z-diff}
Let $(z,y,r)$ satisfy \Cref*{inv:1,inv:2} and, for all $i$, $|r_i| \le (1 - K(1 - \alpha)) \epsilon d(i)^{1-\beta}$. Then, for all $i$, $|z_i-z_i^*| \le \epsilon d(i)^{1-\beta}$.
\end{lemma}

To prove this lemma, we use the following result, whose proof is given in \Cref{appendix:proofs:nl}. The proof uses that $T$ shrinks distances by the factor $K(1-\alpha)$ in the degree-scaled infinity norm, and that $z + r = T(z)$ by \Cref*{inv:1,inv:2}.

\begin{restatable}{lemma}{lemRToZDiff} \label{lem:r-to-z-diff}
Let $(z,y,r)$ satisfy \Cref*{inv:1,inv:2}. Then, for every $i$, $|z_i - z^*_i| \le \frac{d(i)^{1 - \beta}}{1 - K(1 - \alpha)} \|D^{\beta-1}r\|_\infty$.
\end{restatable}

\begin{proof}[Proof of \Cref{lem:r-small-to-z-diff}]
The assumption implies $\|D^{\beta-1}r\|_\infty \le (1 - K(1 - \alpha)) \epsilon$. The claim now follows from \Cref{lem:r-to-z-diff}.
\end{proof}

\subsection{Residual Propagation}
In this section, we define the push operator. This operator moves the residual $r_i$ into $z_i$, which reduces the residual at node $i$.

The push operator is shown in \Cref{alg:push-operator}. It takes as input a node $i$, the graph $G$, and a state $(z,y,r)$ that is assumed to satisfy \Cref*{inv:1,inv:2}. It returns an updated state $(z,y,r)$ that also satisfies \Cref*{inv:1,inv:2}. The operator runs in $O(d(i))$ time.

\begin{algorithm}[H]
\caption{\textsc{Push}($i, G, z, y, r$)} \label{alg:push-operator}
\begin{algorithmic}[1]
\REQUIRE Node $i$, graph $G$, state $(z, y, r)$
\ENSURE Updated state $(z,y,r)$
\STATE $\Delta \leftarrow r_i$
\STATE $z_i \leftarrow z_i + \Delta$
\FOR{each $j \in N(i)$}
    \STATE $y_j \leftarrow y_j + (1-\alpha)w_{ji}\Delta$
    \STATE $r_j \leftarrow f_j(y_j)-z_j$
\ENDFOR
\STATE $r_i \leftarrow f_i(y_i)-z_i$
\STATE \textbf{return} $(z,y,r)$
\end{algorithmic}
\end{algorithm}

We prove the following lemma for this procedure.

\begin{restatable}{lemma}{lemPushInvZ} \label{lem:push-inv-z}
    Suppose \textsc{Push}$(i,G,z,y,r)$ is called on a state $(z,y,r)$ that satisfies \Cref*{inv:1,inv:2}, and let $(z',y',r')$ be the new state. Then $(z',y',r')$ also satisfies \Cref*{inv:1,inv:2}. Moreover, for all $j \not= i$,
    \begin{math}
        |r'_j| \le |r_j| + K(1-\alpha) w_{ji}|r_i|.
    \end{math}
    Furthermore,
    \begin{math}
        |r'_i| \le K(1-\alpha)w_{ii}|r_i|.
    \end{math}
\end{restatable}

The proof is given in \Cref{appendix:proofs:nl}. Combining the push operator with \Cref{lem:r-small-to-z-diff} gives a simple algorithm for computing an approximation $z$: repeatedly push any node $i$ whose residual exceeds the bound in \Cref{lem:r-small-to-z-diff}. When this process terminates, $z$ is a valid approximation of $z^*$. \Cref{alg:cleanup} gives the full procedure.

\begin{algorithm}[H]
\caption{\textsc{Cleanup}($G, z, y, r$)} \label{alg:cleanup}
\begin{algorithmic}[1]
\REQUIRE Graph $G$, state $(z, y, r)$
\ENSURE Updated state $(z,y,r)$ satisfying $|z_i - z^*_i| \le \epsilon d(i)^{1-\beta}$ for all $i$
\WHILE{there exists $i$ such that $|r_i|>(1-K(1-\alpha))\epsilon\, d(i)^{1-\beta}$}
    \STATE $(z,y,r) \leftarrow \textsc{Push}(i,G,z,y,r)$
\ENDWHILE
\STATE \textbf{return} $(z,y,r)$
\end{algorithmic}
\end{algorithm}

\subsection{Graph Updates}
In this section, we give an algorithm to maintain a state $(z,y,r)$ satisfying \cref*{inv:1,inv:2} after a single edge insertion or deletion.

Suppose an edge $(u, v)$ is inserted. Since the degrees of $u$ and $v$ change, the invariants are immediately violated for each neighbor of the two endpoints: the degree of each endpoint changes, and so do the normalized edge weights from their neighbors to them. One approach is to iterate over all such neighbors and update their residuals. This is the approach used by \cite{zheng2022instant} in Algorithm 2. However, this takes time at least $\Omega(d(u)+d(v))$, which prevents efficient dynamic updates. To avoid this cost, we use a scaling trick introduced by \cite{zhang2016approximate}. Define
\begin{align*}
    z'_u = \left(\frac{d(u) + 1}{d(u)}\right)^{1 - \beta} z_u, \qquad
    z'_v = \left(\frac{d(v) + 1}{d(v)}\right)^{1 - \beta} z_v,
\end{align*}
and set $z'_w = z_w$ for all other $w \notin \{u,v\}$. Then, for the new state $(z',y',r')$, \Cref*{inv:1} continues to hold at every node $w\notin\{u,v\}$ if we keep $y'_w=y_w$. Since $z'_w=z_w$ for such nodes, \Cref*{inv:2} also continues to hold there if we keep $r'_w=r_w$. It remains to define $y'_u,y'_v$ and $r'_u,r'_v$ so that \Cref*{inv:1,inv:2} hold at the two endpoints. We show that this can be done in $O(1)$ time, so the full state can be updated in constant time while preserving \Cref*{inv:1,inv:2}. \Cref{alg:update-edge} gives the details.

\begin{algorithm}[H]
\caption{\textsc{UpdateEdge}($G, u, v, type, z, y, r$)} \label{alg:update-edge}
\begin{algorithmic}[1]
\REQUIRE Graph $G$, edge endpoints $u, v$, operation $type$ (\textsc{Insert} or \textsc{Delete}), state vectors $(z,y,r)$
\ENSURE Updated state vectors $(z,y,r)$

\STATE \# Determine edge operation sign
\STATE $\sigma \leftarrow +1$ \textbf{if} $type=\textsc{Insert}$ \textbf{else} $-1$

\STATE \# Compute intermediate values based on current state
\FOR{each endpoint $i \in \{u, v\}$}
    \STATE $d_i \leftarrow d(i)$ \quad \# Store current degree
    \STATE $x_i \leftarrow z_i / d_i^{1-\beta}$
    \STATE $S_i \leftarrow \frac{d_i^\beta}{1-\alpha}(y_i - \alpha s_i)$
\ENDFOR

\STATE \# Apply updates
\FOR{each endpoint $i \in \{u, v\}$ with opposite endpoint $j \in \{v, u\}$}
    \STATE $d'_i \leftarrow d_i + \sigma$ \quad \# Compute new degree
    \STATE $z_i \leftarrow z_i \cdot \left(\frac{d'_i}{d_i}\right)^{1-\beta}$ \quad \# Apply the scaling trick
    \STATE $y_i \leftarrow \alpha s_i + (1-\alpha)\frac{S_i + \sigma x_j}{(d'_i)^\beta}$ \quad \# Add the new neighbor $j$'s contribution
    \STATE $r_i \leftarrow f_i(y_i) - z_i$ \quad \# Update residual state
\ENDFOR

\STATE \textbf{return} $(z,y,r)$
\end{algorithmic}
\end{algorithm}

We end this section by noting that \textsc{UpdateEdge} preserves \Cref*{inv:1,inv:2}. We formally show this in \Cref{lem:update-inv} in the Appendix.

\subsection{Dynamic Algorithm}
Using the procedures defined above, we now give the dynamic algorithm for our non-linear model. We start from an initial state $(z^{\mathrm{start}}_0, y^{\mathrm{start}}_0, r^{\mathrm{start}}_0)$ satisfying \Cref*{inv:1,inv:2}, defined by $z^{\mathrm{start}}_0 := 0$, $y^{\mathrm{start}}_0 := \alpha s$, and $r^{\mathrm{start}}_0 := F(\alpha s)$. We then apply \textsc{Cleanup} to obtain the post-initial state $(z_0, y_0, r_0)$, whose residuals are small enough to give the desired approximation guarantee. After this initialization step, each edge update is handled by first applying \textsc{UpdateEdge} and then applying \textsc{Cleanup}. \Cref{alg:dynamic-update} gives the full procedure.

\begin{algorithm}[H]
\caption{Dynamic Graph Update} \label{alg:dynamic-update}
\begin{algorithmic}[1]
\REQUIRE Initial graph $G_0$, edge update events $\{\Event_1, \dots, \Event_k\}$
\ENSURE Approximation $z_t$ of the fixed point after each update $t=1,\dots,k$, satisfying $|z_{t,i}-z^*_{t,i}| \le \epsilon\, d_t(i)^{1-\beta}$ for all $i$
\STATE $z^{\mathrm{start}}_0 \leftarrow 0$, \quad $y^{\mathrm{start}}_0 \leftarrow \alpha s$, \quad $r^{\mathrm{start}}_0 \leftarrow F(\alpha s)$
\STATE $(z_0,y_0,r_0) \leftarrow \textsc{Cleanup}(G_0, z^{\mathrm{start}}_0, y^{\mathrm{start}}_0, r^{\mathrm{start}}_0)$
\FOR{each $\Event_t = ((u,v), type_t)$ for $t \in \{1,\dots,k\}$}
    \STATE $(z'_t, y'_t, r'_t) \leftarrow \textsc{UpdateEdge}(G_{t-1}, u, v, type_t, z_{t-1}, y_{t-1}, r_{t-1})$
    \STATE Update $G_{t-1}$ to $G_t$ using $\Event_t$
    \STATE $(z_t, y_t, r_t) \leftarrow \textsc{Cleanup}(G_t, z'_t, y'_t, r'_t)$
\ENDFOR
\STATE \textbf{return} $(z_k,y_k,r_k)$
\end{algorithmic}
\end{algorithm}

\subsection{Time Analysis}
We are now ready to analyze the running time of \Cref{alg:dynamic-update}. We begin by introducing the potential function $\Psi_W(r)$, following \citet{zhang2016approximate} and \citet{zheng2022instant}. The function is defined as
\begin{equation} \label{eq:psi-def}
    \Psi_W(r) := \frac{\|D^\beta r\|_1}{1-K(1-\alpha)}
    = \frac{1}{1-K(1-\alpha)}\sum_{j=1}^n d(j)^\beta |r_j|.
\end{equation}
We keep the subscript $W$ to indicate that this is the potential for the current graph. However, the potential depends on $W$ only through the current degree matrix $D$.

The potential is built so that each push pays for itself. A push at node $i$ removes the residual $r_i$, but it adds new residual of at most $K(1-\alpha)w_{ji}|r_i|$ at each neighbor $j$ (\Cref{lem:push-inv-z}). With the weights $d(j)^\beta$, these additions sum to at most $K(1-\alpha)d(i)^\beta|r_i|$, because $\sum_j d(j)^\beta w_{ji} = d(i)^\beta$ (\Cref{lem:weighted-column-sum} in the Appendix). So the sum $\sum_j d(j)^\beta|r_j|$ drops by at least $(1-K(1-\alpha))d(i)^\beta|r_i|$, and the factor $1/(1-K(1-\alpha))$ in $\Psi_W$ turns this drop into $d(i)^\beta|r_i|$. \textsc{Cleanup} pushes $i$ only when $|r_i| > (1-K(1-\alpha))\epsilon\, d(i)^{1-\beta}$, so the drop is at least $(1-K(1-\alpha))\epsilon\, d(i)$. This is proportional to the $O(d(i))$ cost of the push.

We now bound the change in potential caused by applying the $\textsc{Push}$ procedure to a node $i$. The proof is given in \Cref{appendix:proofs:nl}.
\begin{restatable}{lemma}{lemPushPotential}
\label{lem:push-potential}
Let $(z,y,r)$ satisfy \Cref*{inv:1,inv:2}, and let $(z',y',r')$ be obtained by pushing node $i$. Then
\begin{math}
    \Psi_W(r') \le \Psi_W(r) - d(i)^\beta |r_i|.
\end{math}
\end{restatable}

Summing this drop over all pushes of \textsc{Cleanup}, as explained above, gives the following bound, which does not depend on the degrees beyond the cost of each push.

\begin{restatable}{lemma}{lemCleanupWork} \label{lem:cleanup-work}
Suppose $\textsc{Cleanup}$ is run on a graph $G$ with weight matrix $W$, and let $T_{\text{clean}}$ be its running time. If cleanup starts from residual $r^{\text{start}}$ and ends at residual $r^{\text{end}}$, then
\begin{align*}
T_{\text{clean}}
\le
\frac{\Psi_W(r^{\text{start}})-\Psi_W(r^{\text{end}})}{(1 - K(1 - \alpha)) \epsilon}.
\end{align*}
\end{restatable}

The proof is given in \Cref{appendix:proofs:nl}. With the above lemma in place, we next analyze how edge events affect the potential.

Each edge update raises the potential, and \textsc{Cleanup} then pays for bringing it back down. The update changes the residual only at $u$ and $v$, and the size of this change depends on the values $z_u$ and $z_v$. This is where $B_z$ enters: after \textsc{Cleanup}, $|z_i| \le (B_z+\epsilon)\,d(i)^{1-\beta}$ for every $i$ (\Cref{lem:z-bounded} in the Appendix). The increase in potential can be split into two terms: $\Psi_{W_t}(r'_t-r_{t-1})$ and $\Psi_{W_t}(r_{t-1})-\Psi_{W_{t-1}}(r_{t-1})$. We bound the first term in \Cref{lem:single-edge-increment} and the second term in \Cref{lem:single-edge-increment-sec}. We then combine these bounds in \Cref{thm:update-cost-z}. We give the proofs for these lemmas and theorem in \Cref{appendix:proofs:nl}.

\begin{restatable}{lemma}{lemSingleEdgeIncrement} \label{lem:single-edge-increment}
Consider $(z'_t,y'_t,r'_t)$ from \Cref{alg:dynamic-update}, which is the state obtained after applying $\Event_t$ to $(z_{t-1},y_{t-1},r_{t-1})$. Let $W_{t-1},D_{t-1}$ be the weight and degree matrices before this event, and let $W_t,D_t$ be the corresponding matrices after the event. Then
\begin{align*}
\Psi_{W_t}(r'_t - r_{t-1})
\le
\frac{2\left(1+2K(1-\alpha)\right)}{1-K(1-\alpha)} (B_z+\epsilon).
\end{align*}
\end{restatable}

\begin{restatable}{lemma}{lemSingleEdgeIncrementSec} \label{lem:single-edge-increment-sec}
Consider $(z'_t,y'_t,r'_t)$ from \Cref{alg:dynamic-update}, which is the state obtained after applying $\Event_t$ to $(z_{t-1},y_{t-1},r_{t-1})$. Let $W_{t-1},D_{t-1}$ be the weight and degree matrices before this event, and let $W_t,D_t$ be the corresponding matrices after the event. Then
\begin{equation*}
\Psi_{W_t}(r_{t-1}) - \Psi_{W_{t-1}}(r_{t-1}) \le 2\epsilon.
\end{equation*}
\end{restatable}

\begin{restatable}[Cost of one edge update]{theorem}{thmUpdateCostZ}
\label{thm:update-cost-z}
Let $T_t$ be the total work of the update at time $t$. Then
\begin{align*}
T_t
&\le
\frac{2}{1-K(1-\alpha)}
+
\frac{2\left(1+2K(1-\alpha)\right)(B_z+\epsilon)}{\left(1-K(1-\alpha)\right)^2\epsilon}
+
\frac{\Psi_{W_{t-1}}(r_{t-1})-\Psi_{W_t}(r_t)}{\left(1-K(1-\alpha)\right)\epsilon}.
\end{align*}
\end{restatable}

We now combine the preceding results to obtain the final upper bound theorem. The proof is given in \Cref{appendix:proofs:nl}.

\begin{restatable}[Total cost]{theorem}{thmTotalCostZ}
\label{thm:total-cost-z}
Consider a sequence of $k$ updates $\Event_1, \dots, \Event_k$, and let $T_{\mathrm{total}}$ denote the total running time of \Cref{alg:dynamic-update}. Then
\begin{align*}
T_{\mathrm{total}} \le 
T_{\mathrm{init}}
+
\frac{2k}{1-K(1-\alpha)}
+
\frac{2\left(1+2K(1-\alpha)\right)(B_z+\epsilon)}{\left(1-K(1-\alpha)\right)^2\epsilon} k,
\end{align*}
where $T_{\mathrm{init}} = \frac{\|D_0^\beta F(\alpha s)\|_1}{\left(1-K(1-\alpha)\right)^2\epsilon}$.
\end{restatable}

Our analysis of this algorithm is asymptotically tight. The following theorem shows that the amortized time of our algorithm is at least $\Omega(1/\epsilon)$, even in the linear case. We provide the proof in \Cref{appendix:proofs:nl}.

\begin{restatable}[Lower bound on total cost]{theorem}{thmPushLowerBound}
For any sufficiently small $\epsilon > 0$ and any $k$, there exists an initial graph with $n = \Theta(1/\epsilon)$ nodes, a vector $s$ with $\|s\|_\infty \le 1$, and a sequence of $k$ edge insertions and deletions such that \Cref{alg:dynamic-update} requires a total work of $\Omega(k/\epsilon)$. This result holds even in the linear case where $f_i(x) = x$ for all $i$.
\end{restatable}

%% file: sections/4_exact.tex
\section{Exact Dynamic Algorithm For the Linear Case}
Our main result concerns approximate maintenance for nonlinear updates, but for completeness, we also provide an exact dynamic result for the linear case. When $F$ is the identity map, the fixed point is $z^* = [I-(1-\alpha)W]^{-1}\alpha s$. The algorithm adapts the dynamic matrix inverse framework of \citet{sankowski2004dynamic} to the normalized matrix $W=D^{-\beta}AD^{\beta-1}$. Unlike an entry update, one edge update changes the endpoint degrees and hence affects multiple rows and columns of $W$. We show how to represent these changes as low-rank updates and how to incorporate them into the lazy update and rebuild scheme of \citet{sankowski2004dynamic}. The details are deferred to \Cref{appendix:proofs:exact}. We write $\omega$ for the exponent of square matrix multiplication. More generally, we write $\omega(a,b,c)$ for the exponent of the running time for multiplying an $n^a \times n^b$ matrix by an $n^b \times n^c$ matrix.

\begin{theorem}[Exact dynamic maintenance for the linear case]
Assume $f_i(x)=x$ for all $i$, and suppose every graph has permanent self-loops. The fixed point $z^*$ can be maintained explicitly under (possibly weighted) edge insertions and deletions, with an initialization cost $O(n^\omega)$ and worst-case update time
$O\left(n^{\omega(1,\gamma,1)-\gamma} + n^{1+\gamma}\right)$
for any parameter $\gamma \in [0,1]$. Using current rectangular matrix multiplication bounds, this gives update time
$O(n^{1.5275})$.
\end{theorem}

\begin{proof}[Proof sketch]
Let $B=I-(1-\alpha)W$. An edge update changes only two rows and two columns of $W$, and can be written as a rank-$4$ update. We then apply the lazy update and periodic rebuild scheme of \citet{sankowski2004dynamic}. Between rebuilds, $B^{-1}$ is maintained implicitly using the Woodbury identity, while the explicit vector $z^*=B^{-1}\alpha s$ is updated after each edge change. After $k$ edge updates, one update costs $O(nk+k^2)$ time. Following \citet{sankowski2004dynamic}, after $n^\gamma$ updates we rebuild the inverse using the fact that only $O(n^\gamma)$ rows and columns have changed. This gives amortized update time $O(n^{\omega(1,\gamma,1)-\gamma}+n^{1+\gamma})$. The standard deamortization used by \citet{sankowski2004dynamic} gives the same worst-case bound. Full details are in \Cref{appendix:proofs:exact}.
\end{proof}

%% file: sections/5_experiments.tex
\section{Experiments}
We note that our main contribution is the theoretical analysis of the dynamic algorithm (\Cref{alg:dynamic-update}). Still, we run some experiments to evaluate the practicality of our approach. We organize the experimental contribution into two parts, corresponding to the two main benefits of the model: accuracy and running time. First, in the static setting, we show that adding non-linearity improves predictive accuracy across several benchmark datasets. Second, in the dynamic setting, we evaluate whether our algorithm preserves efficient update performance and compare its running time under non-linearity with the linear algorithm of \cite{zheng2022instant}. We publicly provide our code.
\footnote{\url{https://anonymous.4open.science/r/dynamic-nl-public-code-A225}}

Although the model and analysis are presented for scalar node values, all experiments use multi-dimensional node features. We apply the propagation independently to each feature dimension, using the same graph and parameters for every dimension. All non-linear activations reported in the experiments are uniformly bounded by a constant, and $K$-Lipschitz with $K \le 1$.

\subsection{Non-linearity Improves Accuracy} \label{section:exp-acc}

We evaluate our approach on several common static datasets. The datasets are explained in \Cref{appendix:exp-acc:datasets}. Following \citet{zheng2022instant}, we first propagate features and then train an MLP on the propagated representations. We use the same activation at every node and pass both the propagated vector $z$ and the pre-activation vector $y=\alpha s+(1-\alpha)Wz$ to the classifier. Hyperparameters are selected by validation accuracy using the same search space for all activation functions. We use three parameterized activation functions: (i) scaled tanh, $\tanh(cx)/c$ with parameter $c$, denoted by \enquote{Sc Tanh $c$}; (ii) shifted tanh, $\tanh(x-c)$ with parameter $c$, denoted by \enquote{Sh Tanh $c$}; and (iii) hard tanh, $\min(c, \max(-c, x))$ with parameter $c$, denoted by \enquote{H Tanh $c$}.

We compare with GCN~\citep{kipf2017semi}, GAT~\citep{velivckovic2017graph}, Geom-GCN~\citep{pei2020geom}, APPNP~\citep{gasteiger2018predict}, ACM-GCN~\citep{luan2022revisiting}, and GREAD~\citep{choi2023gread}, all of which we run ourselves on the same splits as our model. \citet{platonov2023critical} show that \textsc{Chameleon} and \textsc{Squirrel} contain duplicated nodes and release filtered versions without them, so we also evaluate on these two filtered datasets. \Cref{table:exp-acc:acc} reports the results on the original datasets, and \Cref{table:exp-acc:filtered} in the Appendix reports the results on the filtered ones. On every dataset, at least one non-linear activation is more accurate than the linear model, and the gain is sometimes large, such as 11 points on \textsc{Wisconsin}. On \textsc{PubMed} and on both filtered datasets, our non-linear model is more accurate than every baseline.

\begin{table}[H]
\caption{Classification accuracy of models, with standard errors. The best method is \underline{underlined}, and the best activation function within our model is shown in \textcolor{green!50!black}{green}.} 
\centering
\label{table:exp-acc:acc}
\setlength{\tabcolsep}{4pt}
\begin{adjustbox}{width=\linewidth}
\input{tables/exp_acc}
\end{adjustbox}
\end{table}

\subsection{Dynamic Graphs}
\begin{figure*}[thbp]
\centering
\begin{subfigure}[t]{0.33\textwidth}
\includegraphics[width=\linewidth]{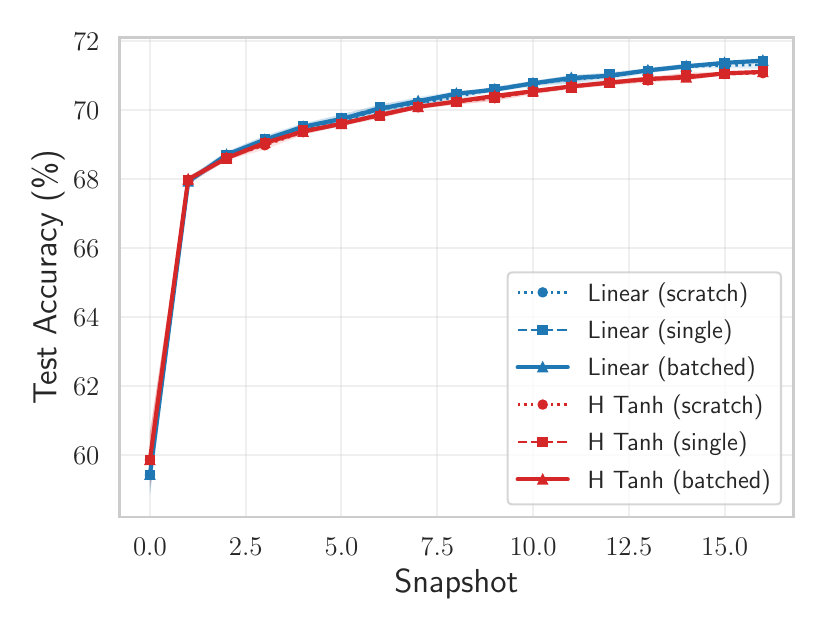}
\end{subfigure}\hfill
\begin{subfigure}[t]{0.33\textwidth}
\includegraphics[width=\linewidth]{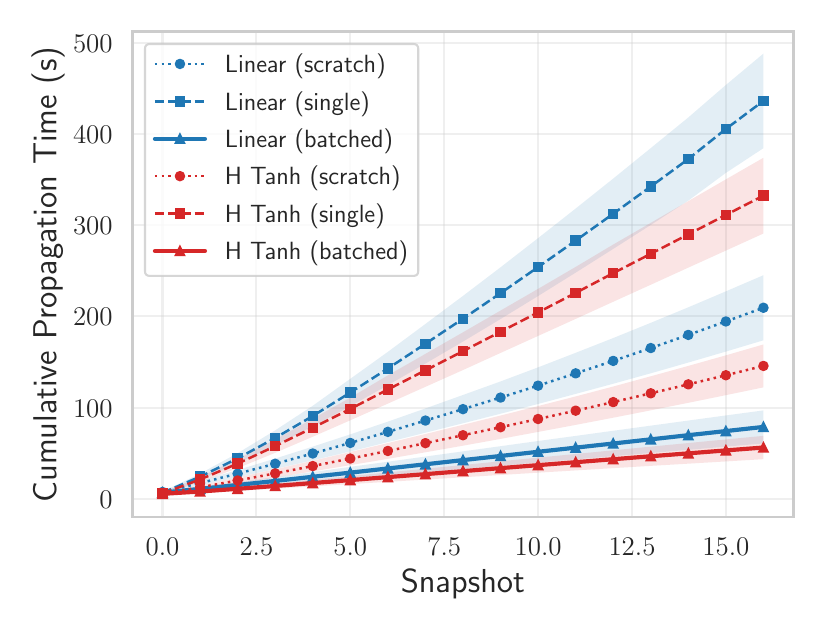}
\end{subfigure}\hfill
\begin{subfigure}[t]{0.33\textwidth}
\includegraphics[width=\linewidth]{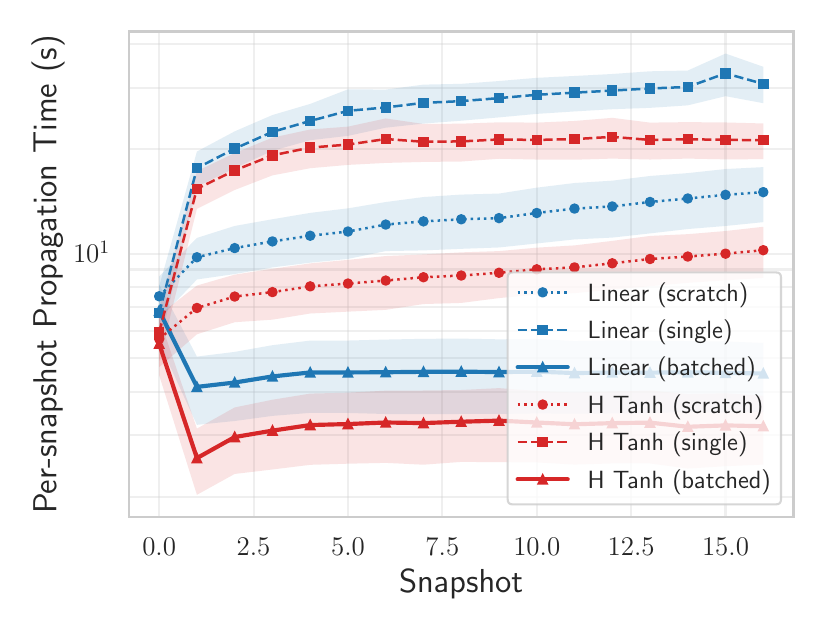}
\end{subfigure}\hfill
\caption{Results on the dynamic \texttt{arxiv} dataset. Accuracy is similar for the linear and hard-tanh models. Since this is the smallest dynamic experiment, recomputation from scratch is still faster than the \enquote{single} strategy, which handles the harder online setting of propagating after each edge update. Batched updates are much faster than both.}
\label{fig:arxiv-comparison}
\end{figure*}

We evaluate dynamic propagation on \texttt{arxiv}, \texttt{products}, and \texttt{SBM-500K} \citep{hu2020open,zheng2022instant}. The first two are incremental edge-addition streams from Open Graph Benchmark, while \texttt{SBM-500K} is a synthetic dynamic graph by \citet{zheng2022instant}, and contains both insertions and deletions. We compare the linear model with our non-linear hard-tanh model using $c=2.5$. Propagation is parallelized over feature dimensions using 16 CPU cores. For the linear baseline, we use the identity activation function within our framework which yields the same linear propagation model as \citet{zheng2022instant}. We evaluate three update strategies: \enquote{scratch}, which recomputes the propagation state after each snapshot; \enquote{single}, which applies our dynamic algorithm after each individual edge update; and \enquote{batched}, which processes all edge changes within a snapshot prior to cleanup. The \enquote{single} strategy aligns with the online update model analyzed in our theorems, whereas \enquote{batched} serves as a practical heuristic at the snapshot level.

\Cref{fig:arxiv-comparison,fig:products-comparison,fig:sbm500k-comparison} show that the linear and non-linear models have similar propagation time and accuracy across all datasets. The dynamic methods are especially beneficial on the larger graphs, where recomputation from scratch is much more expensive. The per-snapshot propagation time is nearly stable after the first snapshot, and cumulative propagation time grows approximately linearly with the number of snapshots, consistent with our amortized update guarantees. All curves in these experiments are averaged over 10 independent runs, and the shaded regions show 95\% confidence intervals. Additional experiments in \Cref{appendix:exp-dyn} show that the observed running times follow the predicted dependence on $\epsilon$ and $\alpha$. We also run larger versions of the SBM benchmark with 2M and 5M nodes, and the \texttt{papers100M} graph from Open Graph Benchmark \citep{hu2020open}, which has 111M nodes and 1.6B edges (\Cref{table:exp-dyn:large} in the Appendix). The advantage of dynamic updates over recomputation grows with the size of the graph. On \texttt{papers100M}, a batched update takes about two minutes per snapshot, while recomputation takes over five hours.

\begin{figure*}[thbp]
\centering
\begin{subfigure}[t]{0.33\textwidth}
\includegraphics[width=\linewidth]{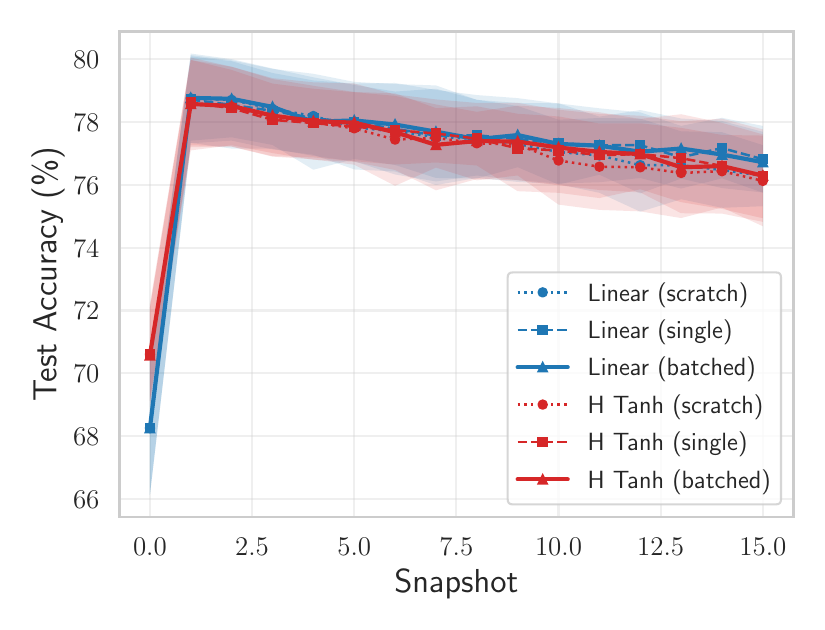}
\end{subfigure}\hfill
\begin{subfigure}[t]{0.33\textwidth}
\includegraphics[width=\linewidth]{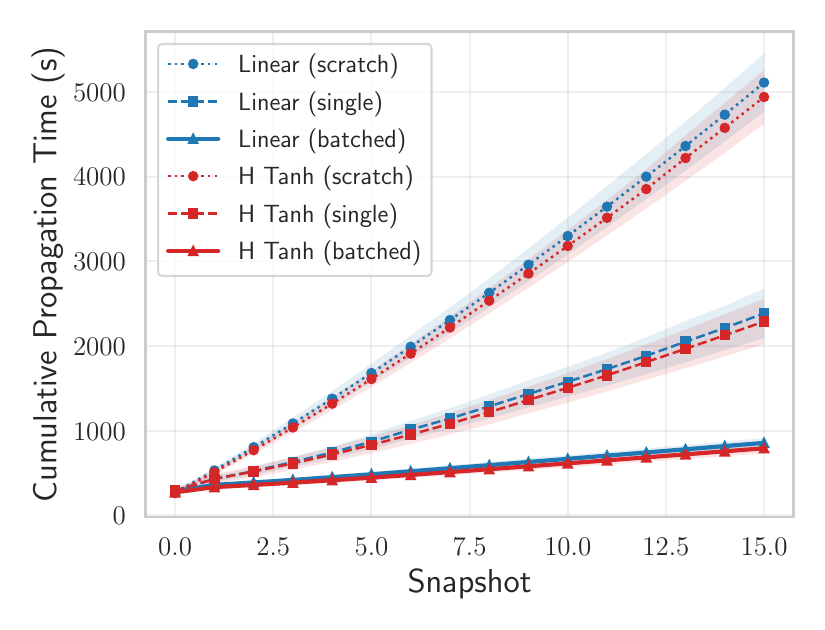}
\end{subfigure}\hfill
\begin{subfigure}[t]{0.33\textwidth}
\includegraphics[width=\linewidth]{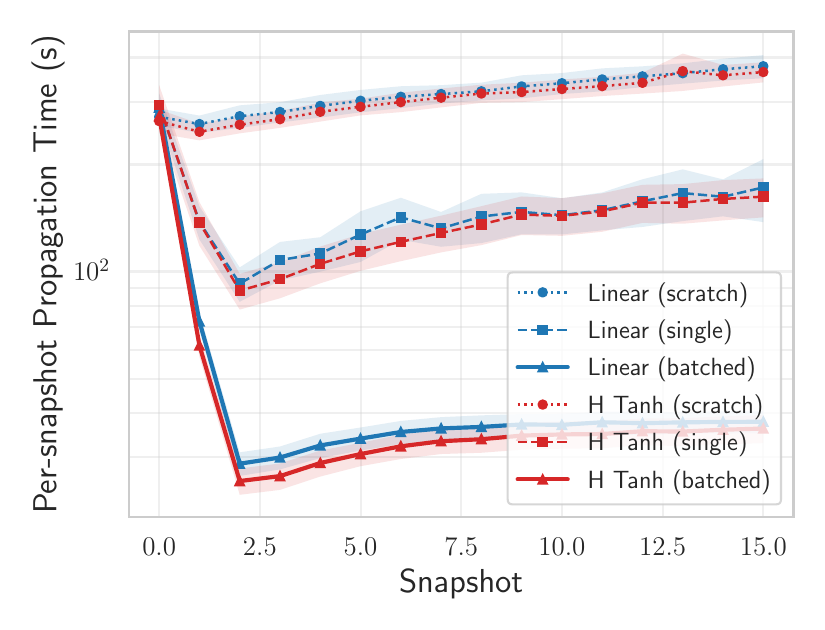}
\end{subfigure}\hfill
\caption{Results on the dynamic \texttt{products} dataset. Accuracy peaks early and then slowly decreases across later snapshots. Unlike on \texttt{arxiv}, recomputation from scratch becomes much more expensive on this larger graph, while dynamic updates keep the propagation cost substantially lower.}
\label{fig:products-comparison}
\end{figure*}

\begin{figure*}[thbp]
\centering
\begin{subfigure}[t]{0.33\textwidth}
\includegraphics[width=\linewidth]{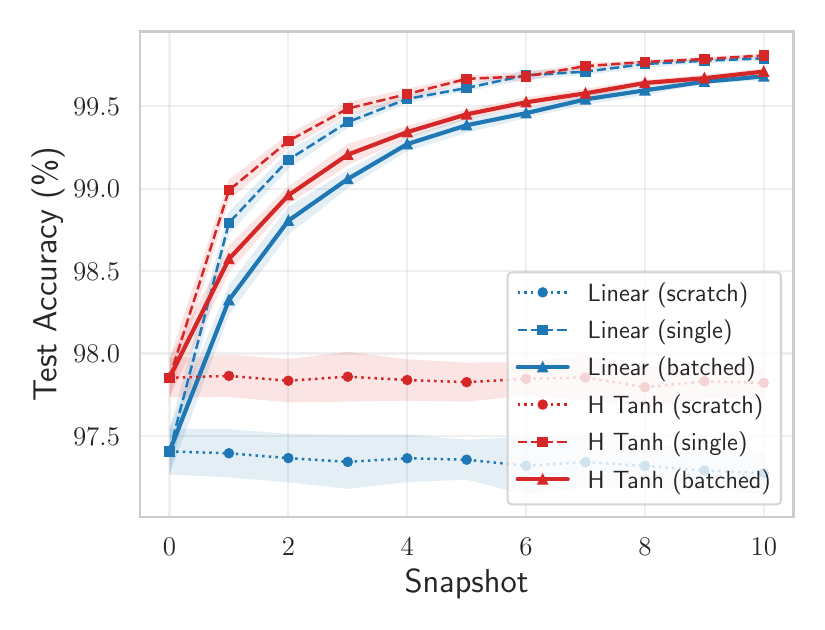}
\end{subfigure}\hfill
\begin{subfigure}[t]{0.33\textwidth}
\includegraphics[width=\linewidth]{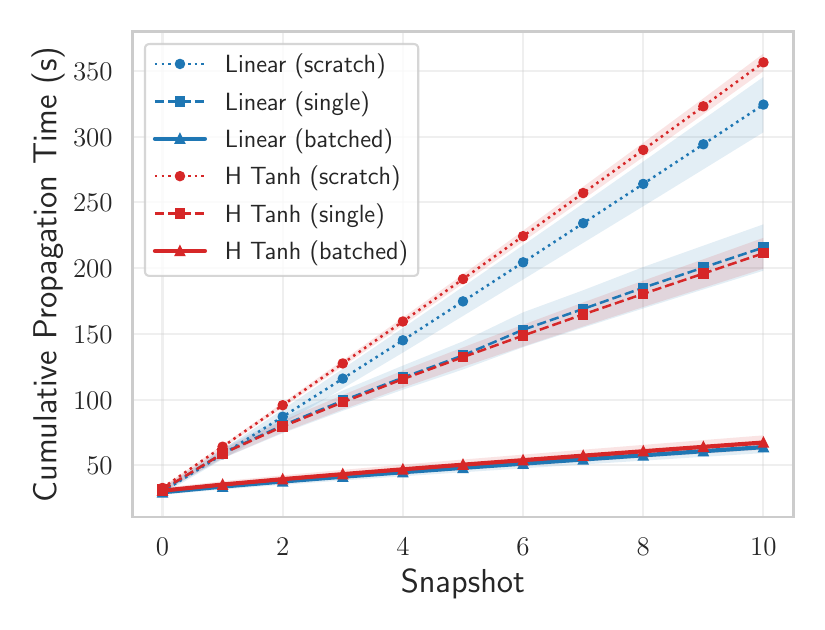}
\end{subfigure}\hfill
\begin{subfigure}[t]{0.33\textwidth}
\includegraphics[width=\linewidth]{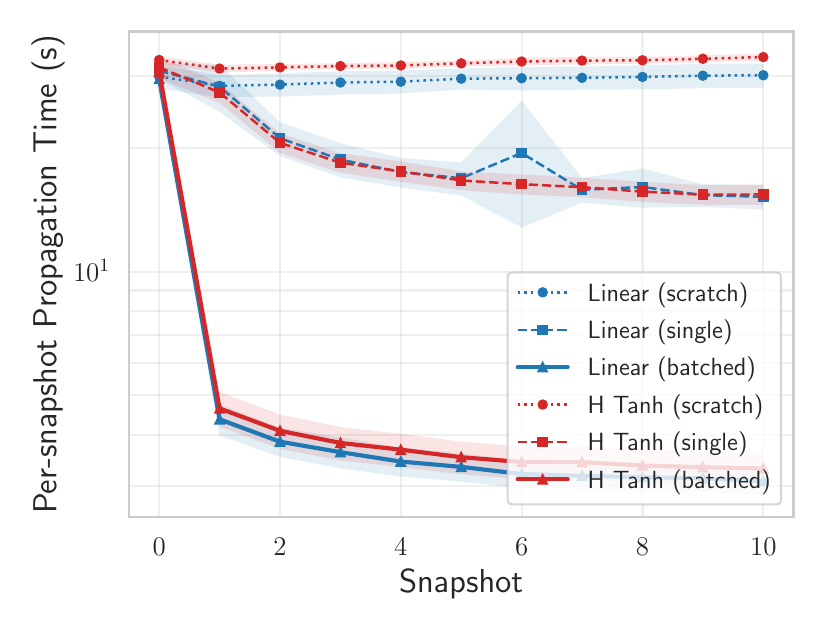}
\end{subfigure}\hfill
\caption{Results on the dynamic \texttt{SBM-500K} dataset. The dynamic methods remain efficient even with both edge insertions and deletions. After the first snapshot, their per-snapshot cost stays low and stable while preserving the accuracy behavior of the linear and hard-tanh models.}
\label{fig:sbm500k-comparison}
\end{figure*}

%% file: tables/exp_acc.tex
\begin{tabular}{lccccccccc}
\toprule
Method & \textsc{Corn.} & \textsc{Texas} & \textsc{Wisc.} & \textsc{Squi.} & \textsc{Actor} & \textsc{Cham.} & \textsc{Cora} & \textsc{Cite.} & \textsc{Pubm.} \\
\midrule
GCN & 42.2{\scriptsize\,$\pm$\,1.9} & 60.5{\scriptsize\,$\pm$\,1.6} & 54.4{\scriptsize\,$\pm$\,1.6} & 30.1{\scriptsize\,$\pm$\,0.3} & 28.8{\scriptsize\,$\pm$\,0.2} & 42.9{\scriptsize\,$\pm$\,0.8} & 82.6{\scriptsize\,$\pm$\,0.2} & 71.4{\scriptsize\,$\pm$\,0.5} & 79.0{\scriptsize\,$\pm$\,0.2} \\
GAT & 50.6{\scriptsize\,$\pm$\,2.0} & 61.2{\scriptsize\,$\pm$\,2.0} & 54.7{\scriptsize\,$\pm$\,1.6} & 30.8{\scriptsize\,$\pm$\,0.5} & 29.9{\scriptsize\,$\pm$\,0.4} & 46.0{\scriptsize\,$\pm$\,0.6} & 83.0{\scriptsize\,$\pm$\,0.3} & \underline{72.4{\scriptsize\,$\pm$\,0.3}} & 77.3{\scriptsize\,$\pm$\,0.1} \\
Geom-GCN & 63.1{\scriptsize\,$\pm$\,1.9} & 68.2{\scriptsize\,$\pm$\,1.1} & 65.5{\scriptsize\,$\pm$\,1.1} & 37.4{\scriptsize\,$\pm$\,0.3} & 31.8{\scriptsize\,$\pm$\,0.3} & 59.9{\scriptsize\,$\pm$\,0.9} & 66.0{\scriptsize\,$\pm$\,0.3} & 65.1{\scriptsize\,$\pm$\,0.3} & 76.6{\scriptsize\,$\pm$\,0.2} \\
APPNP & 69.1{\scriptsize\,$\pm$\,1.7} & 76.1{\scriptsize\,$\pm$\,1.4} & 80.3{\scriptsize\,$\pm$\,1.5} & 34.3{\scriptsize\,$\pm$\,0.5} & 35.5{\scriptsize\,$\pm$\,0.3} & 51.2{\scriptsize\,$\pm$\,0.7} & \underline{83.8{\scriptsize\,$\pm$\,0.2}} & 71.6{\scriptsize\,$\pm$\,0.2} & 79.7{\scriptsize\,$\pm$\,0.2} \\
ACM-GCN & \underline{77.8{\scriptsize\,$\pm$\,1.0}} & \underline{85.6{\scriptsize\,$\pm$\,1.2}} & \underline{85.9{\scriptsize\,$\pm$\,1.1}} & 54.1{\scriptsize\,$\pm$\,0.4} & 36.3{\scriptsize\,$\pm$\,0.4} & 67.3{\scriptsize\,$\pm$\,0.6} & 82.0{\scriptsize\,$\pm$\,0.2} & 69.9{\scriptsize\,$\pm$\,0.4} & 79.2{\scriptsize\,$\pm$\,0.3} \\
GREAD & 73.9{\scriptsize\,$\pm$\,0.7} & 85.5{\scriptsize\,$\pm$\,1.2} & 85.1{\scriptsize\,$\pm$\,1.3} & \underline{58.2{\scriptsize\,$\pm$\,0.3}} & \underline{37.2{\scriptsize\,$\pm$\,0.3}} & \underline{70.5{\scriptsize\,$\pm$\,0.5}} & 81.6{\scriptsize\,$\pm$\,0.3} & 71.0{\scriptsize\,$\pm$\,0.2} & 78.3{\scriptsize\,$\pm$\,0.2} \\
\midrule
Linear & 66.5{\scriptsize\,$\pm$\,1.5} & 68.8{\scriptsize\,$\pm$\,1.5} & 70.8{\scriptsize\,$\pm$\,1.5} & 41.7{\scriptsize\,$\pm$\,0.4} & 33.8{\scriptsize\,$\pm$\,0.2} & 55.1{\scriptsize\,$\pm$\,0.5} & 79.8{\scriptsize\,$\pm$\,0.1} & 66.3{\scriptsize\,$\pm$\,0.0} & 80.2{\scriptsize\,$\pm$\,0.1} \\
Sc Tanh 0.5 & 69.2{\scriptsize\,$\pm$\,1.0} & 74.7{\scriptsize\,$\pm$\,1.8} & 79.8{\scriptsize\,$\pm$\,1.5} & \textcolor{green!50!black}{41.8{\scriptsize\,$\pm$\,0.5}} & 35.4{\scriptsize\,$\pm$\,0.3} & 56.1{\scriptsize\,$\pm$\,0.4} & 80.5{\scriptsize\,$\pm$\,0.0} & \textcolor{green!50!black}{70.7{\scriptsize\,$\pm$\,0.0}} & 81.6{\scriptsize\,$\pm$\,0.0} \\
Sc Tanh 8 & 70.0{\scriptsize\,$\pm$\,0.7} & 76.0{\scriptsize\,$\pm$\,2.0} & 80.4{\scriptsize\,$\pm$\,1.2} & 39.1{\scriptsize\,$\pm$\,0.5} & 35.3{\scriptsize\,$\pm$\,0.3} & 53.3{\scriptsize\,$\pm$\,0.8} & \textcolor{green!50!black}{83.3{\scriptsize\,$\pm$\,0.0}} & 69.5{\scriptsize\,$\pm$\,0.1} & \underline{\textcolor{green!50!black}{82.3{\scriptsize\,$\pm$\,0.0}}} \\
Sh Tanh -1.2 & \textcolor{green!50!black}{70.7{\scriptsize\,$\pm$\,1.1}} & \textcolor{green!50!black}{77.4{\scriptsize\,$\pm$\,1.7}} & \textcolor{green!50!black}{81.8{\scriptsize\,$\pm$\,0.9}} & 36.6{\scriptsize\,$\pm$\,0.5} & \textcolor{green!50!black}{35.6{\scriptsize\,$\pm$\,0.2}} & 52.3{\scriptsize\,$\pm$\,0.6} & 52.8{\scriptsize\,$\pm$\,0.1} & 46.1{\scriptsize\,$\pm$\,0.2} & 71.8{\scriptsize\,$\pm$\,0.0} \\
H Tanh 2.5 & 67.7{\scriptsize\,$\pm$\,1.4} & 76.2{\scriptsize\,$\pm$\,1.8} & 79.2{\scriptsize\,$\pm$\,1.6} & 41.7{\scriptsize\,$\pm$\,0.5} & 35.4{\scriptsize\,$\pm$\,0.3} & \textcolor{green!50!black}{56.2{\scriptsize\,$\pm$\,0.4}} & 81.6{\scriptsize\,$\pm$\,0.1} & 70.4{\scriptsize\,$\pm$\,0.0} & 82.3{\scriptsize\,$\pm$\,0.0} \\
\bottomrule
\end{tabular}

%% file: sections/6_conclusion.tex
\section{Conclusion}
We studied dynamic maintenance of GNN representations with non-linear propagation. We gave a residual-based push algorithm with an amortized $O(1/\epsilon)$ update bound under a degree-normalized error guarantee, and we also gave an exact dynamic algorithm for the linear case. Our work has several limitations. First, the propagation is fixed and has no learned weights, so our results do not directly apply to the most popular GNNs. We leave maintaining a model with a learned matrix that mixes the features open. Second, on large dynamic graphs we compare only our linear and non-linear models. Third, we do not theoretically analyze the batched updates used in our experiments. A natural next step is to experiment more with this non-linear propagation model, including more activation functions and learned activation parameters.

%% file: sections/appendix.tex
\section{Omitted Lemmas and Proofs of Section 2} \label{appendix:proofs:nl}

We restate the lemmas omitted from the main text and provide their proofs. We also introduce several auxiliary lemmas used in later proofs. The results are ordered by their logical dependency and their sequence in the main text.

\Cref{table:params} lists the parameters used in the paper and their roles.

\begin{table}[htbp]
\centering
\caption{Parameters used in the paper.}
\label{table:params}
\begin{tabular}{lp{0.72\linewidth}}
\toprule
Symbol & Meaning \\
\midrule
$n$ & Number of nodes. \\
$A$ & Adjacency matrix of the current graph, with a self-loop at every node. \\
$D$ & Diagonal degree matrix of the current graph. \\
$d(i)$ & Degree of node $i$. Since every node has a self-loop, $d(i) \ge 1$. \\
$\beta \in [0,1]$ & Normalization parameter in $W = D^{-\beta}AD^{\beta-1}$. \\
$W$ & Normalized weight matrix. \\
$\alpha \in (0,1)$ & Teleportation probability, the weight a node keeps on its own input at each step. \\
$s$ & Input values. \\
$f_i$, $F$ & Activation at node $i$. $F$ applies $f_i$ to coordinate $i$. \\
$K \le 1$ & Every $f_i$ is $K$-Lipschitz. \\
$T$, $z^*$ & The map $T(z) = F(\alpha s + (1-\alpha)Wz)$ and its unique fixed point. \\
$\|\cdot\|_{D,\infty}$ & Degree-scaled infinity norm, $\|v\|_{D,\infty} = \|D^{\beta-1}v\|_\infty$. \\
$B_z$ & Bound on $\|z^*\|_{D,\infty}$ on every graph. \\
$\epsilon$ & Accuracy. The algorithm keeps $|z_i - z^*_i| \le \epsilon\, d(i)^{1-\beta}$ for every $i$. \\
$z$, $y$, $r$ & Maintained state: the estimate of $z^*$, the pre-activation vector $\alpha s + (1-\alpha)Wz$, and the residual $F(y) - z$. \\
$\Psi_W$ & Potential $\sum_j d(j)^\beta |r_j| / (1-K(1-\alpha))$, used to bound the running time. \\
$k$ & Number of edge updates. \\
$\omega$, $\omega(a,b,c)$, $\gamma$ & Matrix multiplication exponents and the rebuild parameter of the exact algorithm in \Cref{appendix:proofs:exact}. \\
\bottomrule
\end{tabular}
\end{table}

\begin{lemma} \label{lem:zs-bounded-01}
    If $f_i(x) \in [-B_z,B_z]$ for all $i$ and $x$, then $\| D^{\beta-1} z^* \|_\infty \le B_z$.
\end{lemma}
\begin{proof}
    Since $f_i(x) \in [-B_z,B_z]$ for all $i$ and $x$, each coordinate of $F(x)$ also lies in $[-B_z,B_z]$. Because
    \begin{math}
        z^* = F(\alpha s + (1-\alpha)Wz^*),
    \end{math}
    it follows that $z_i^* \in [-B_z,B_z]$ for every $i$. Since $d(i) \ge 1$ and $\beta \in [0,1]$, we have $d(i)^{\beta-1} \le 1$. Hence
    \begin{equation*}
        \|D^{\beta-1}z^*\|_\infty
        =
        \max_i d(i)^{\beta-1}|z_i^*|
        \le
        \max_i |z_i^*|
        \le B_z. \qedhere
    \end{equation*}
\end{proof}

\lemDegScaledNormW*
\begin{proof}
Since $W=D^{-\beta}AD^{\beta-1}$,
\begin{align*}
    \|Wv\|_{D,\infty}
    &=
    \|D^{\beta-1}D^{-\beta}AD^{\beta-1}v\|_\infty \\
    &=
    \|D^{-1}AD^{\beta-1}v\|_\infty \\
    &\le
    \|D^{-1}A\|_\infty \|v\|_{D,\infty}.
\end{align*}
Moreover,
\begin{math}
    \|D^{-1}A\|_\infty
    =
    \max_i d(i)^{-1}\sum_j A_{ij}
    =
    1.
\end{math}
Thus $\|Wv\|_{D,\infty}\le \|v\|_{D,\infty}$.
\end{proof}

The next lemma bounds $B_z$ for activations that may be unbounded, such as ReLU.

\begin{lemma} \label{lem:zs-bounded-general}
    The fixed point $z^*$ satisfies
    \begin{equation*}
        \|D^{\beta-1}z^*\|_\infty \le \frac{\max_i |f_i(0)| + K\alpha\|s\|_\infty}{1-K(1-\alpha)}.
    \end{equation*}
\end{lemma}
\begin{proof}
    Fix a node $i$, and let $x_i = \alpha s_i + (1-\alpha)(Wz^*)_i$, so that $z^*_i = f_i(x_i)$. Since $f_i$ is $K$-Lipschitz,
    \begin{equation*}
        |z^*_i| \le |f_i(0)| + K|x_i| \le \max_j |f_j(0)| + K\alpha|s_i| + K(1-\alpha)|(Wz^*)_i|.
    \end{equation*}
    We multiply both sides by $d(i)^{\beta-1}$, which is at most $1$ because $d(i) \ge 1$ and $\beta \le 1$. The first term stays at most $\max_j |f_j(0)|$. The second term becomes at most $K\alpha\|s\|_\infty$. The third term becomes at most $K(1-\alpha)\|Wz^*\|_{D,\infty}$, which is at most $K(1-\alpha)\|z^*\|_{D,\infty}$ by \Cref{lem:deg-scaled-norm-W}. Taking the maximum over $i$ gives
    \begin{equation*}
        \|z^*\|_{D,\infty} \le \max_j |f_j(0)| + K\alpha\|s\|_\infty + K(1-\alpha)\|z^*\|_{D,\infty}.
    \end{equation*}
    Since $K(1-\alpha) < 1$, rearranging gives the claim.
\end{proof}

\lemUniqueFixedPoint*
\begin{proof}
For any $x,y\in\R^n$,
\begin{align*}
    \|T(x)-T(y)\|_{D,\infty}
    &=
    \|F(\alpha s+(1-\alpha)Wx)-F(\alpha s+(1-\alpha)Wy)\|_{D,\infty}.
\end{align*}
Since each $f_i$ is $K$-Lipschitz, the pointwise map $F$ is also $K$-Lipschitz under $\|\cdot\|_{D,\infty}$. Therefore,
\begin{align*}
    \|T(x)-T(y)\|_{D,\infty}
    &\le
    K(1-\alpha)\|W(x-y)\|_{D,\infty}.
\end{align*}
By \Cref{lem:deg-scaled-norm-W},
\begin{math}
    \|T(x)-T(y)\|_{D,\infty}
    \le
    K(1-\alpha)\|x-y\|_{D,\infty}.
\end{math}
Since $\alpha\in(0,1)$ and $K\le 1$, we have $K(1-\alpha)<1$. Hence $T$ is a contraction under the degree-scaled infinity norm. The Banach fixed-point theorem \citep{banach1922operations} implies that the fixed point is unique and that the iteration $z(t+1)=T(z(t))$ converges to it from any initial point.
\end{proof}

\lemRToZDiff*
\begin{proof}
For any $x,y\in\R^n$,
\begin{align*}
\|F(x)-F(y)\|_{D,\infty}
&=
\max_i \frac{|f_i(x_i)-f_i(y_i)|}{d(i)^{1-\beta}} \\
&\le
K\|x-y\|_{D,\infty},
\end{align*}
because each $f_i$ is $K$-Lipschitz. Together with \Cref{lem:deg-scaled-norm-W}, this gives
\begin{math}
    \|T(x)-T(y)\|_{D,\infty}
    \le
    K(1-\alpha)\|x-y\|_{D,\infty}.
\end{math}
By \Cref*{inv:1,inv:2}, we have
\begin{math}
    z+r=F(y)=T(z).
\end{math}
Subtracting this identity from $z^*=T(z^*)$ gives
\begin{math}
    z-z^*=T(z)-T(z^*)-r.
\end{math}
Thus
\begin{align*}
\|z-z^*\|_{D,\infty}
&\le
\|T(z)-T(z^*)\|_{D,\infty}
+
\|r\|_{D,\infty} \\
&\le
K(1-\alpha)\|z-z^*\|_{D,\infty}
+
\|r\|_{D,\infty}.
\end{align*}
Rearranging,
\begin{math}
    \|z-z^*\|_{D,\infty}
    \le
    \frac{\|r\|_{D,\infty}}{1-K(1-\alpha)}.
\end{math}
Therefore, for each $i$,
\begin{equation*}
|z_i-z_i^*|
\le
\frac{d(i)^{1-\beta}}{1-K(1-\alpha)}
\|D^{\beta-1}r\|_\infty. \qedhere
\end{equation*}
\end{proof}

\lemPushInvZ*
\begin{proof}
Only coordinate $i$ of $z$ changes, and $z'_i-z_i=r_i$. Hence, for every $j$,
\begin{align*}
y'_j
&=
y_j+(1-\alpha)w_{ji}r_i \\
&=
\alpha s_j+(1-\alpha)\sum_k w_{jk}z'_k.
\end{align*}
Thus \Cref*{inv:1} is preserved. By the definition of $r'$,
\begin{math}
    z'_j+r'_j=f_j(y'_j)
\end{math}
for every $j$, so \Cref*{inv:2} is also preserved.

For $j\neq i$,
\begin{equation*}
    |r'_j|
    = |f_j(y'_j)-z_j|
    \le |f_j(y'_j)-f_j(y_j)| + |f_j(y_j)-z_j|
    = K(1-\alpha)w_{ji}|r_i| + |r_j|.
\end{equation*}
For $i$, using $z_i+r_i=f_i(y_i)$,
\begin{equation*}
    |r'_i|
    = |f_i(y'_i)-z_i-r_i|
    = |f_i(y'_i)-f_i(y_i)|
    \le K|y'_i-y_i|
    = K(1-\alpha)w_{ii}|r_i|.
\end{equation*}
This proves the claim.
\end{proof}

\begin{restatable}{lemma}{lemUpdateInv} \label{lem:update-inv}
    Suppose $\textsc{UpdateEdge}(G,u,v,type,z,y,r)$ is called on a state $(z,y,r)$ satisfying \Cref*{inv:1,inv:2}, and let $(z',y',r')$ be the returned state. Then $(z',y',r')$ satisfies \Cref*{inv:1,inv:2}.
\end{restatable}
\begin{proof}
Let $d'_u=d_u+\sigma$ and $d'_v=d_v+\sigma$ be the new degrees. By \Cref*{inv:1}, before the update,
\begin{math}
S_u
=
\frac{d_u^\beta}{1-\alpha}(y_u-\alpha s_u)
=
\sum_{j\in N(u)} \frac{z_j}{d(j)^{1-\beta}},
\end{math}
and the same identity holds for $S_v$.

The rescaling step preserves the degree-scaled values of the endpoints:
\begin{equation*}
    \frac{z'_u}{(d'_u)^{1-\beta}}
    =
    \frac{z_u}{d_u^{1-\beta}}
    =
    x_u,
    \qquad
    \frac{z'_v}{(d'_v)^{1-\beta}}
    =
    \frac{z_v}{d_v^{1-\beta}}
    =
    x_v.
\end{equation*}
Thus the contributions of $u$ and $v$ to each non-endpoint neighbor do not change. Hence, for every $w\notin\{u,v\}$, \Cref*{inv:1} continues to hold with $y'_w=y_w$. Since $z'_w=z_w$ and $r'_w=r_w$, \Cref*{inv:2} also continues to hold at such nodes.

It remains to check the endpoints. After the update, the scaled neighbor sum at $u$ is $S_u+\sigma x_v$. Therefore
\begin{math}
    y'_u=\alpha s_u+(1-\alpha)(S_u+\sigma x_v)/(d_u+\sigma)^\beta
\end{math}
is exactly the value required by \Cref*{inv:1} at $u$. The same argument applies to $v$. Finally, setting
\begin{math}
    r'_u=f_u(y'_u)-z'_u
\end{math}
and
\begin{math}
    r'_v=f_v(y'_v)-z'_v
\end{math}
gives \Cref*{inv:2} at the two endpoints. Therefore the returned state satisfies both invariants.
\end{proof}

\begin{restatable}{lemma}{lemWeightedColumnSum} \label{lem:weighted-column-sum}
For every node $i$,
\begin{math}
\sum_j d(j)^\beta w_{ji} = d(i)^\beta.
\end{math}
\end{restatable}
\begin{proof}
By the definition of $W$,
\begin{math}
d(j)^\beta w_{ji}
=
d(j)^\beta \cdot \frac{A_{ji}}{d(j)^\beta d(i)^{1-\beta}}
=
\frac{A_{ji}}{d(i)^{1-\beta}}.
\end{math}
Summing over $j$ gives
\begin{math}
\sum_j d(j)^\beta w_{ji}
=
\frac{\sum_j A_{ji}}{d(i)^{1-\beta}}
=
\frac{d(i)}{d(i)^{1-\beta}}
=
d(i)^\beta.
\end{math}
\end{proof}

\lemPushPotential*
\begin{proof}
First, note that $\left(1-K(1-\alpha)\right)\Psi_W(r') = \sum_{j=1}^n d(j)^\beta |r'_j|$. Applying \Cref{lem:push-inv-z} to $r'$ gives
\begin{math}
\sum_{j=1}^n d(j)^\beta |r'_j|
\le
\sum_{j\neq i} d(j)^\beta\left(|r_j|+K(1-\alpha) w_{ji}|r_i|\right) +
d(i)^\beta K(1-\alpha) w_{ii}|r_i|.
\end{math}
The right-hand side is equal to
\begin{math}
\sum_{j\neq i} d(j)^\beta |r_j| + K(1-\alpha)|r_i| \sum_{j=1}^n d(j)^\beta w_{ji}.
\end{math}
By \Cref{lem:weighted-column-sum}, and after adding and subtracting $d(i)^\beta |r_i|$, we obtain
\begin{align*}
\left(1-K(1-\alpha)\right)\Psi_W(r')
&\le
\left(1-K(1-\alpha)\right)\Psi_W(r)
-
d(i)^\beta |r_i|
+
K(1-\alpha) d(i)^\beta |r_i| \\
&=
\left(1-K(1-\alpha)\right)\Psi_W(r)
-
\left(1-K(1-\alpha)\right)d(i)^\beta |r_i|.
\end{align*}
Dividing by $1-K(1-\alpha)$ gives the claim.
\end{proof}

\lemCleanupWork*
\begin{proof}
First note that $T_{\text{clean}} = \sum_{\text{pushes at } i} d(i)$. By \Cref{lem:push-potential}, a push at node $i$ decreases $\Psi_W$ by at least $d(i)^\beta |r_i|$. Since $\textsc{Cleanup}$ only pushes when $|r_i|>(1 - K(1 - \alpha))\epsilon d(i)^{1-\beta}$, the decrease from this push is at least
\begin{math}
d(i)^\beta |r_i|
>
(1 - K(1 - \alpha))\epsilon d(i).
\end{math}
Summing this bound over all pushes gives the result.
\end{proof}

\begin{restatable}{lemma}{lemDegreePerturbation}
\label{lem:degree-perturbation}
Let $d>0$, let $\delta\in \R$ satisfy $d+\delta>0$, and let $\beta\in[0,1]$. Then
\begin{math}
d^{1-\beta} |(d+\delta)^\beta-d^\beta|
\le
|\delta|.
\end{math}
Equivalently,
\begin{math}
|(d+\delta)^\beta-d^\beta|
\le
\frac{|\delta|}{d^{1-\beta}}.
\end{math}
In particular, when $\delta=1$,
\begin{math}
d^{1-\beta}\left((d+1)^\beta-d^\beta\right)\le 1.
\end{math}
\end{restatable}
\begin{proof}
Let $x:=\delta/d$. Then $x>-1$, and
\begin{equation*}
d^{1-\beta}|(d+\delta)^\beta-d^\beta|
=
d^{1-\beta}d^\beta |(1+x)^\beta-1|
=
d |(1+x)^\beta-1|.
\end{equation*}
It is enough to show that
\begin{math}
    |(1+x)^\beta-1|\le |x|
\end{math}
for all $x>-1$.

If $x\ge 0$, Bernoulli's inequality gives $(1+x)^\beta\le 1+x$. Therefore,
\begin{math}
    0\le (1+x)^\beta-1\le x=|x|.
\end{math}
If $-1<x\le 0$, set $y:=1+x\in(0,1]$. Since $\beta\in[0,1]$, we have $y^\beta\ge y$. Hence
\begin{math}
    0\le 1-y^\beta\le 1-y=-x=|x|.
\end{math}
Equivalently,
\begin{math}
    |(1+x)^\beta-1|\le |x|.
\end{math}
Thus
\begin{math}
    d |(1+x)^\beta-1| \le d|x|=|\delta|,
\end{math}
which proves the claim.
\end{proof}

\begin{lemma} \label{lem:z-bounded}
Suppose $z$ satisfies $\|z - z^*\|_{D,\infty} \le \epsilon$. Then $\|z\|_{D,\infty} \le B_z + \epsilon$.
\end{lemma}

\begin{proof}
By assumption, $\|z^*\|_{D,\infty} \le B_z$. Therefore, $\|z\|_{D,\infty} \le \|z^*\|_{D,\infty} + \|z-z^*\|_{D,\infty} \le B_z + \epsilon$.
\end{proof}

\lemSingleEdgeIncrement*

\begin{proof}
Let $\sigma\in\{+1,-1\}$, where $\sigma=+1$ denotes an insertion and $\sigma=-1$ denotes a deletion. Write
\begin{math}
    a=d_{t-1}(u)
\end{math}
and
\begin{math}
    b=d_{t-1}(v).
\end{math}
Then $d_t(u)=a+\sigma$ and $d_t(v)=b+\sigma$. Since self-loops are permanent and only the non-self edge $(u,v)$ is updated, both $a+\sigma$ and $b+\sigma$ are positive. Define
\begin{math}
    x_i:=d_{t-1}(i)^{\beta-1}z_{t-1,i}
\end{math}
and
\begin{math}
    X:=\|z_{t-1}\|_{D_{t-1},\infty}=\|x\|_\infty.
\end{math}
Since \textsc{Cleanup} has been applied before event $t$, \Cref{lem:z-bounded} gives $X\le B_z+\epsilon$. By \Cref{alg:update-edge}, the vector $r'_t-r_{t-1}$ is supported only on $u$ and $v$. In the procedure,
\begin{math}
    S_v=\sum_{j\in N_{t-1}(v)} x_j
\end{math}
and
\begin{math}
    S_u=\sum_{j\in N_{t-1}(u)} x_j.
\end{math}
Thus $|S_v|\le bX$ and $|S_u|\le aX$. We prove the bound for $v$; the proof for $u$ is the same.

Before the update,
\begin{math}
    y_{t-1,v}=\alpha s_v+(1-\alpha)b^{-\beta}S_v.
\end{math}
After the update,
\begin{math}
    y'_v=\alpha s_v+(1-\alpha)(b+\sigma)^{-\beta}(S_v+\sigma x_u).
\end{math}
Therefore,
\begin{equation*}
y'_v-y_{t-1,v}
=
(1-\alpha)
\left[
\left(\frac{1}{(b+\sigma)^\beta}-\frac{1}{b^\beta}\right)S_v
+
\frac{\sigma x_u}{(b+\sigma)^\beta}
\right].
\end{equation*}

By \Cref*{inv:2},
\begin{math}
r'_{t,v}-r_{t-1,v}
=
\left(f_v(y'_v)-z'_v\right)
-
\left(f_v(y_{t-1,v})-z_{t-1,v}\right).
\end{math}
Using the $K$-Lipschitz property of $f_v$,
\begin{math}
|r'_{t,v}-r_{t-1,v}|
\le
|z'_v-z_{t-1,v}|+K|y'_v-y_{t-1,v}|.
\end{math}
Multiplying by $(b+\sigma)^\beta$ gives
\begin{align*}
(b+\sigma)^\beta |r'_{t,v}-r_{t-1,v}|
&\le
(b+\sigma)^\beta |z'_v-z_{t-1,v}| \\
&\quad+
K(1-\alpha)(b+\sigma)^\beta
\left|
\frac{1}{(b+\sigma)^\beta}
-
\frac{1}{b^\beta}
\right|
|S_v| \\
&\quad+
K(1-\alpha)|x_u|.
\end{align*}

We bound the three terms on the right. Since
\begin{math}
    z'_v=(b+\sigma)^{1-\beta}x_v
\end{math}
and
\begin{math}
    z_{t-1,v}=b^{1-\beta}x_v,
\end{math}
the first term is
\begin{equation*}
(b+\sigma)^\beta
\left|
(b+\sigma)^{1-\beta}
-
b^{1-\beta}
\right|
|x_v|.
\end{equation*}
Applying \Cref{lem:degree-perturbation} with $d=b+\sigma$, $\delta=-\sigma$, and exponent $1-\beta$ gives
\begin{math}
    (b+\sigma)^\beta
    |(b+\sigma)^{1-\beta}-b^{1-\beta}|
    \le 1.
\end{math}
Thus the first term is at most $X$. For the second term,

\begin{align*}
(b+\sigma)^\beta\left|\frac{1}{(b+\sigma)^\beta}-\frac{1}{b^\beta}\right||S_v|
\le\frac{|(b+\sigma)^\beta-b^\beta|}{b^\beta}bX
= b^{1-\beta}|(b+\sigma)^\beta-b^\beta| X.
\end{align*}
Applying \Cref{lem:degree-perturbation} with $d=b$, $\delta=\sigma$, and exponent $\beta$ gives
\begin{math}
    b^{1-\beta}|(b+\sigma)^\beta-b^\beta|\le 1.
\end{math}
Thus the second term is at most $K(1-\alpha)X$. 
The third term is also at most $K(1-\alpha)X$, since $|x_u|\le X$. Combining these bounds,
\begin{equation*}
d_t(v)^\beta |r'_{t,v}-r_{t-1,v}|
=
(b+\sigma)^\beta |r'_{t,v}-r_{t-1,v}|
\le
\left(1+2K(1-\alpha)\right)X.
\end{equation*}

The same bound holds for $u$. Hence
\begin{align*}
\Psi_{W_t}(r'_t-r_{t-1})
&=
\frac{
d_t(u)^\beta |r'_{t,u}-r_{t-1,u}|
+
d_t(v)^\beta |r'_{t,v}-r_{t-1,v}|
}
{1-K(1-\alpha)} \\
&\le
\frac{2\left(1+2K(1-\alpha)\right)}{1-K(1-\alpha)}X.
\end{align*}
Using $X\le B_z+\epsilon$ gives the claim.
\end{proof}

\lemSingleEdgeIncrementSec*
\begin{proof}
By the definition of $\Psi_W$,
\begin{align*}
\Psi_{W_t}(r_{t-1})-\Psi_{W_{t-1}}(r_{t-1})
&=
\frac{1}{1-K(1-\alpha)}
\sum_i
\left(d_t(i)^\beta-d_{t-1}(i)^\beta\right)
|r_{t-1}(i)|.
\end{align*}
Only the endpoints $u$ and $v$ can change degree. If the event is a deletion, then $d_t(i)\le d_{t-1}(i)$ for $i\in\{u,v\}$, so the above quantity is non-positive. It remains to consider an insertion. For $i\in\{u,v\}$, we have $d_t(i)=d_{t-1}(i)+1$. Since cleanup has terminated before the update,
\begin{math}
|r_{t-1}(i)|
\le
\left(1-K(1-\alpha)\right)\epsilon d_{t-1}(i)^{1-\beta}.
\end{math}
Therefore,
\begin{align*}
\Psi_{W_t}(r_{t-1})-\Psi_{W_{t-1}}(r_{t-1})
&\le
\epsilon
\sum_{i\in\{u,v\}}
d_{t-1}(i)^{1-\beta}
\left(
(d_{t-1}(i)+1)^\beta
-
d_{t-1}(i)^\beta
\right).
\end{align*}

By \Cref{lem:degree-perturbation}, applied with $d=d_{t-1}(i)$ and $\delta=1$,
\begin{equation*}
d_{t-1}(i)^{1-\beta}
\left(
(d_{t-1}(i)+1)^\beta
-
d_{t-1}(i)^\beta
\right)
\le 1.
\end{equation*}
Thus, in all cases,
\begin{math}
\Psi_{W_t}(r_{t-1})-\Psi_{W_{t-1}}(r_{t-1})
\le 2\epsilon.
\end{math}
\end{proof}

\thmUpdateCostZ*
\begin{proof}
By \Cref{lem:cleanup-work} applied to the new graph $G_t$,
\begin{align*}
T_t
\le
\frac{\Psi_{W_t}(r'_t)-\Psi_{W_t}(r_t)}{\left(1-K(1-\alpha)\right)\epsilon}.
\end{align*}
The triangle inequality gives
\begin{math}
\Psi_{W_t}(r'_t)
\le
\Psi_{W_t}(r_{t-1})
+
\Psi_{W_t}(r'_t-r_{t-1}).
\end{math}

Combining this with
\begin{math}
\Psi_{W_t}(r_{t-1})
=
\Psi_{W_{t-1}}(r_{t-1})
+
\left(\Psi_{W_t}(r_{t-1})-\Psi_{W_{t-1}}(r_{t-1})\right),
\end{math}
we obtain
\begin{align*}
T_t
\le
\frac{\Psi_{W_{t-1}}(r_{t-1})-\Psi_{W_t}(r_t)}{\left(1-K(1-\alpha)\right)\epsilon}
+
\frac{\Psi_{W_t}(r'_t-r_{t-1})}{\left(1-K(1-\alpha)\right)\epsilon}
+
\frac{\Psi_{W_t}(r_{t-1})-\Psi_{W_{t-1}}(r_{t-1})}{\left(1-K(1-\alpha)\right)\epsilon}.
\end{align*}
The first term is kept in the final bound. The second term is bounded by \Cref{lem:single-edge-increment}, and the third term is bounded by \Cref{lem:single-edge-increment-sec}. Combining these bounds gives the claim.
\end{proof}

\thmTotalCostZ*

\begin{proof}
The total running time is the sum of the initialization cost and the cost of all edge updates. By \Cref{lem:cleanup-work}, the initialization cost is at most
\begin{equation*}
    \frac{\Psi_{W_0}(r^{\mathrm{start}}_0)-\Psi_{W_0}(r_0)}
    {\left(1-K(1-\alpha)\right)\epsilon}.
\end{equation*}
For the update costs, summing the bound of \Cref{thm:update-cost-z} over $t=1,\dots,k$ gives
\begin{align*}
\sum_{t=1}^k T_t
&\le
\frac{2k}{1-K(1-\alpha)}
+
\frac{2\left(1+2K(1-\alpha)\right)}
{\left(1-K(1-\alpha)\right)^2\epsilon}
\sum_{t=1}^k
\|D_{t-1}^{\beta-1}z_{t-1}\|_\infty \\
&\qquad
+
\frac{1}{\left(1-K(1-\alpha)\right)\epsilon}
\sum_{t=1}^k
\left(
\Psi_{W_{t-1}}(r_{t-1})
-
\Psi_{W_t}(r_t)
\right).
\end{align*}
The potential terms telescope:
\begin{align*}
\sum_{t=1}^k
\left(
\Psi_{W_{t-1}}(r_{t-1})
-
\Psi_{W_t}(r_t)
\right)
=
\Psi_{W_0}(r_0)-\Psi_{W_k}(r_k)
\le
\Psi_{W_0}(r_0).
\end{align*}

Therefore, including initialization,
\begin{align*}
T_{\mathrm{total}}
&\le
\frac{\Psi_{W_0}(r^{\mathrm{start}}_0)-\Psi_{W_0}(r_0)}
{\left(1-K(1-\alpha)\right)\epsilon}
+
\frac{\Psi_{W_0}(r_0)}
{\left(1-K(1-\alpha)\right)\epsilon}
+
\frac{2k}{1-K(1-\alpha)} \\
&\quad+
\frac{2\left(1+2K(1-\alpha)\right)}
{\left(1-K(1-\alpha)\right)^2\epsilon}
\sum_{t=1}^k
\|D_{t-1}^{\beta-1}z_{t-1}\|_\infty.
\end{align*}
Also,
\begin{math}
    \Psi_{W_0}(r^{\mathrm{start}}_0)
    =
    \frac{\|D_0^\beta F(\alpha s)\|_1}{1-K(1-\alpha)}.
\end{math}
By \Cref{lem:z-bounded},
\begin{math}
    \|D_{t-1}^{\beta-1}z_{t-1}\|_\infty \le B_z+\epsilon
\end{math}
for every $t$. Substituting this bound gives
\begin{equation*}
T_{\mathrm{total}}
\le
T_{\mathrm{init}}
+
\frac{2k}{1-K(1-\alpha)}
+
\frac{2\left(1+2K(1-\alpha)\right)(B_z+\epsilon)}
{\left(1-K(1-\alpha)\right)^2\epsilon}k. \qedhere
\end{equation*}
\end{proof}

\thmPushLowerBound*
\begin{proof}
We give a construction in which the same edge is inserted and deleted repeatedly. Each insertion forces a push at a node of degree $\Theta(1/\epsilon)$.

Set $f_i(x)=x$, $\alpha=1/2$, and $\beta=0$. Then $K=1$, and the cleanup threshold at node $i$ is $|r_i|>\frac{1}{2}\epsilon d(i)$. We have $W=AD^{-1}$, so a node $j$ contributes $z_j/d(j)$ to each of its neighbors.

Our graph consists of a star and an isolated vertex. Let $u$ be the center of the star, and let its degree be
\begin{align*}
    d(u)=m:=\left\lfloor \frac{1}{100\epsilon}\right\rfloor
\end{align*}
including its permanent self-loop. Thus $u$ has $m-1$ leaves, and each leaf has degree $2$. Let $v$ be the isolated vertex except for its permanent self-loop, which is initially disconnected from the star. Hence the graph has $n=m+1=\Theta(1/\epsilon)$ nodes.

Define the $s$ vector by $s_v=1$ and $s_i=0$ for all $i\ne v$. In this initial graph, the fixed point has $z^*_v=1$ and $z^*_i=0$ for all $i \ne v$. After cleanup, the maintained state therefore satisfies $|z_u|\le \epsilon m$ and $|z_v-1|\le \epsilon$ and all residuals are below threshold. In particular, $|r_u|\le \frac{1}{2}\epsilon m$ and $|r_v|\le \frac{1}{2}\epsilon$.

Now insert the edge $(u,v)$. We show that the subsequent cleanup must push $u$. The update rule scales $z_u$ by $(m+1)/m$ and adds the new contribution from $v$ to $u$. Hence the new residual at $u$ is
\begin{align*}
    r_u' = r_u+\frac12 z_v-\frac{z_u}{m}.
\end{align*}
Using the bounds above and $\epsilon m\le 1/100$, we get for $\epsilon \le 1/100$,
\begin{align*}
    r_u'
    &\ge -\frac12\epsilon m+\frac12(1-\epsilon)-\epsilon \\
    &\ge -\frac{1}{200} + \frac{1}{2}\left(1-\frac{1}{100}\right) - \frac{1}{100} \\
    &= 0.48.
\end{align*}
On the other hand, the cleanup threshold at $u$ after the insertion is $\frac12\epsilon(m+1)\le0.01$, thus $u$ has a large residual immediately after the edge update. It remains to check that cleanup cannot remove this residual by pushing other nodes first. Before $u$ is pushed, the pushes that can affect $r_u$ are pushes at its neighbors.

First consider node $v$. After the insertion, $d(v)=2$, so pushing $v$ with residual $r'_v$ changes $r_u$ by $(1-\alpha)w_{uv}r'_v=r'_v/4$. The residual left at $v$ after this push is also $r'_v/4$. Therefore, even if $v$ is pushed repeatedly before $u$ is ever pushed, its total possible effect on $r_u$ has magnitude at most
\begin{align*}
    \frac14|r'_v|+\frac1{16}|r'_v|+\frac1{64}|r'_v|+\cdots
    =
    \frac13|r'_v|.
\end{align*}
After the insertion, $r'_v = r_v - z_v +\frac{1}{2}\cdot \frac{z_u}{m}$. Using bounds before the insertion we have $|r'_v| \le \epsilon/2 + (1+\epsilon) + \epsilon/2 = 1 + 2\epsilon$, so all pushes of $v$ can reduce $r_u$ by at most $(1+2\epsilon)/3\le0.34$.

Now consider the previous leaves. The insertion of $(u,v)$ does not change their residuals, because the scaling step preserves the quantity $z_u/d(u)$ that they receive from $u$. Thus each leaf $i$ still has $|r_i|\le \epsilon$, since $d(i)=2$. The same argument as $v$ shows that all pushes of one leaf can affect $r_u$ by at most $|r_i|/3$. Summing over the $m-1$ leaves, the total possible effect from all leaves is at most $\frac{1}{3}(m-1)\epsilon \le \frac{1}{300}$.

Therefore, before the first push of $u$, even after all possible pushes of $v$ and all leaves, the residual at $u$ is still at least $0.48-0.34-\frac1{300}>0.1$. This is larger than the cleanup threshold at $u$, which is below $0.01$. Hence, cleanup cannot terminate before pushing $u$.

Pushing $u$ costs $\Theta(d(u))=\Theta(m)=\Theta(1/\epsilon)$ work. Thus every insertion of $(u,v)$ forces $\Omega(1/\epsilon)$ work.

Finally, take the update sequence that repeatedly toggles the edge $(u,v)$. After every deletion and cleanup, the graph is again the initial graph, so the same argument applies to the next insertion. Therefore, the total work is $\Omega(k/\epsilon)$, and the amortized update time is $\Omega(1/\epsilon)$.
\end{proof}

\section{Exact Dynamic Algorithm For the Linear Case}
\label{appendix:proofs:exact}

In this section, we focus on the linear case, which is closely related to the PageRank problem. We seek an algorithm that computes the exact solution instead of an approximation. As before, we assume that the graph contains permanent self-loops. However, we do not require the graph to be unweighted, and we define the degree of a node as the sum of the weights of its incident edges. Suppose that each $f_i$ is the identity function. Then the unique fixed point of the system is
\begin{equation}
    z^* = \left[I - (1 - \alpha)W\right]^{-1} \alpha s.
\end{equation}

For a static graph, one can compute $z^*$ exactly by computing the matrix inverse $\left[I - (1 - \alpha)W\right]^{-1}$ in $O(n^{\omega})$ time, where $\omega$ is the matrix multiplication exponent \citep{bunch1974triangular}. For a dynamic graph, there is an extensive line of work on dynamic matrix inverse and dynamic systems of linear equations. We adapt these results to the PageRank model, and in particular to the normalized weight matrix. Our algorithms are mainly inspired by the dynamic algorithm for systems of linear equations of \cite{sankowski2004dynamic}. However, their algorithm does not directly give a dynamic algorithm for computing $z^*$ in our setting, because an edge insertion or deletion may change up to two rows and two columns of $W$ due to the normalized weights. Their algorithm does not directly support such updates. One could try to decompose the changes caused by an insertion or deletion into multiple operations, but this requires care, since their algorithm assumes that the matrix remains non-singular after each intermediate update.

In the following lemma, we show that an edge update induces a low-rank update to $W$. This directly follows from the fact that only two rows and two columns can change. However, the specific structure of the update is what allows us to obtain an efficient algorithm.

\begin{lemma} \label{lem:exact:low-rank}
    Suppose an edge between $u,v$ with weight $\delta$ is inserted or deleted from the graph, and let $W_+$ be the normalized weight matrix of the new graph. Then
    \begin{align*}
        W_+ = W + UV^T
    \end{align*}
    for some $U,V \in \R^{n \times 4}$. In other words, the edge update induces an update of rank at most 4 to $W$.
\end{lemma}

\begin{proof}
We use the subscript $+$ to denote quantities after the update. Let the signed change be $\eta=\delta$ for an insertion and $\eta=-\delta$ for a deletion. Let
\begin{math}
    E=\begin{bmatrix}e_u & e_v\end{bmatrix}.
\end{math}
Then
\begin{math}
    A_+=A+\eta(e_ue_v^T+e_ve_u^T)
\end{math}
and
\begin{math}
    D_+=D+\eta(e_ue_u^T+e_ve_v^T).
\end{math}
Let $W=D^{-\beta}AD^{\beta-1}$ and $W_+=D_+^{-\beta}A_+D_+^{\beta-1}$.

Since only the degrees of $u$ and $v$ change, there are diagonal matrices $\Lambda,\Gamma\in\R^{2\times 2}$ such that
\begin{math}
    D_+^{-\beta}D^\beta=I+E\Lambda E^T
\end{math}
and
\begin{math}
    D^{1-\beta}D_+^{\beta-1}=I+E\Gamma E^T.
\end{math}
For $i\in\{u,v\}$, these matrices have entries
\begin{math}
    \Lambda_{ii}=(d_i/d_{i,+})^\beta-1
\end{math}
and
\begin{math}
    \Gamma_{ii}=(d_{i,+}/d_i)^{\beta-1}-1.
\end{math}

Define $H\in\R^{2\times 2}$ by
\begin{math}
    EHE^T
    =
    \eta D_+^{-\beta}(e_ue_v^T+e_ve_u^T)D_+^{\beta-1}.
\end{math}
Equivalently,
\begin{math}
    H_{uv}=\eta d_{u,+}^{-\beta}d_{v,+}^{\beta-1},
\end{math}
\begin{math}
    H_{vu}=\eta d_{v,+}^{-\beta}d_{u,+}^{\beta-1},
\end{math}
and the diagonal entries of $H$ are zero.

Using these definitions,
\begin{equation*}
W_+
=
D_+^{-\beta}A_+D_+^{\beta-1}
=
(I+E\Lambda E^T)W(I+E\Gamma E^T)+EHE^T.
\end{equation*}
Thus
\begin{equation*}
W_+-W
=
E\Lambda E^TW(I+E\Gamma E^T)
+
WE\Gamma E^T
+
EHE^T.
\end{equation*}
Now set
\begin{equation*}
U=
\begin{bmatrix}
E & WE\Gamma
\end{bmatrix}
\in\R^{n\times 4}
\end{equation*}
and
\begin{equation*}
V^T=
\begin{bmatrix}
\Lambda E^TW(I+E\Gamma E^T)+HE^T \\
E^T
\end{bmatrix}
\in\R^{4\times n}.
\end{equation*}
Then $W_+=W+UV^T$, so the update has rank at most $4$.
\end{proof}

Next, we describe at a high level how $z^*$ can be maintained under updates. To do this, we use the Woodbury matrix identity \citep{woodbury1950inverting}, which gives a formula for the inverse of a matrix under a low-rank update.

Let $B := I - (1 - \alpha)W$. Following \citet{sankowski2004dynamic}, we maintain $B^{-1}$ in an implicit form. Specifically, after $k$ updates, we write $B = B_0 + \mathcal{U}\mathcal{V}^T$, where $\mathcal{U},\mathcal{V} \in \R^{n \times 4k}$ and $B_0$ is the last rebuilt matrix. We maintain $M := B_0^{-1}$, $Y := M\mathcal{U}$, and the small inverse $G := (I_{4k} + \mathcal{V}^TY)^{-1}$. We also maintain the current solution $z^* = B^{-1}\alpha s$. By the Woodbury identity,
\begin{equation}
    B^{-1} = M - YG\mathcal{V}^TM.
\end{equation}
We use the subscript $+$ to denote the state after a single edge update. Consider an update from $W$ to $W_+$. By \Cref{lem:exact:low-rank}, $W_+ = W + U_WV_W^T$ for some $U_W,V_W \in \R^{n \times 4}$. Then $B_+ = B + UV^T$, where $U = U_W$ and $V = -(1-\alpha)V_W$. Define $Z := B^{-1}U$ and $S:= I_4 + V^TZ$. Applying the Woodbury identity again gives $B_+^{-1} = B^{-1} - ZS^{-1}V^TB^{-1}$. Since $S$ has size $4 \times 4$, its inverse can be computed in $O(1)$ time. The new solution is therefore $z^*_+ = z^* - ZS^{-1}V^Tz^*$. Once $Z$ is known, this update takes $O(n)$ time. We give the details for computing $Z$ in the lemma below.

\begin{lemma} \label{lem:exact:update}
Let an edge update produce $B_+=B+UV^T$ with $U,V\in\R^{n\times 4}$. Then the maintained state and the solution $z^*_+=B_+^{-1}\alpha s$ can be updated in time $O(nk+k^2)$.
\end{lemma}

\begin{proof}
Before the update, we maintain
\begin{math}
    B=B_0+\mathcal{U}\mathcal{V}^T,
\end{math}
where $\mathcal{U},\mathcal{V}\in\R^{n\times 4k}$,
\begin{math}
    M=B_0^{-1},
\end{math}
\begin{math}
    Y=M\mathcal{U},
\end{math}
\begin{math}
    G=(I+\mathcal{V}^TY)^{-1},
\end{math}
and
\begin{math}
    z^*=B^{-1}\alpha s.
\end{math}
Thus
\begin{equation*}
    B^{-1}=M-YG\mathcal{V}^TM.
\end{equation*}

By \Cref{lem:exact:low-rank}, the edge update has
\begin{math}
    U=\begin{bmatrix}E & WE\Gamma\end{bmatrix},
\end{math}
where $E=\begin{bmatrix}e_u & e_v\end{bmatrix}$ and $\Gamma\in\R^{2\times 2}$ is diagonal. We first compute $Z=B^{-1}U$. Since $B=I-(1-\alpha)W$,
\begin{math}
    B^{-1}W=\frac{1}{1-\alpha}(B^{-1}-I).
\end{math}
Therefore,
\begin{equation*}
Z
=
\begin{bmatrix}
B^{-1}E &
\frac{1}{1-\alpha}(B^{-1}E-E)\Gamma
\end{bmatrix}.
\end{equation*}
It remains to compute $B^{-1}E$. Using the maintained inverse representation,
\begin{equation*}
B^{-1}E
=
ME-YG\mathcal{V}^TME.
\end{equation*}
The matrix $ME$ consists of two columns of $M$, so it can be obtained in $O(n)$ time. Computing $\mathcal{V}^TME$ costs $O(nk)$, multiplying by $G$ costs $O(k^2)$, and multiplying by $Y$ costs $O(nk)$. Hence $Z$ can be computed in $O(nk+k^2)$ time.

Now define
\begin{math}
    S=I_4+V^TZ.
\end{math}
Since $S$ is $4\times 4$, computing $S^{-1}$ takes $O(1)$ time. By the Woodbury identity,
\begin{equation*}
B_+^{-1}
=
B^{-1}-ZS^{-1}V^TB^{-1}.
\end{equation*}
Therefore,
\begin{math}
    z^*_+=z^*-ZS^{-1}V^Tz^*.
\end{math}
The product $V^Tz^*$ costs $O(n)$ time, and the remaining operations also cost $O(n)$ time. This is within the claimed bound.

It remains to update the maintained state. We append the new factors:
\begin{math}
    \mathcal{U}_+=\begin{bmatrix}\mathcal{U} & U\end{bmatrix}
\end{math}
and
\begin{math}
    \mathcal{V}_+=\begin{bmatrix}\mathcal{V} & V\end{bmatrix}.
\end{math}
We also append $MU$ to $Y$. The first two columns, $ME$, are already available. For the remaining two columns, $MWE\Gamma$, use
\begin{math}
    B_0=I-(1-\alpha)W_0
\end{math}
and
\begin{math}
    B=B_0+\mathcal{U}\mathcal{V}^T=I-(1-\alpha)W.
\end{math}
These identities imply
\begin{math}
    W=W_0-\frac{1}{1-\alpha}\mathcal{U}\mathcal{V}^T.
\end{math}
Thus
\begin{equation*}
MWE
=
\frac{1}{1-\alpha}(ME-E)
-
\frac{1}{1-\alpha}Y\mathcal{V}^TE.
\end{equation*}
Therefore,
\begin{equation*}
MU
=
\begin{bmatrix}
ME &
\left(
\frac{1}{1-\alpha}(ME-E)
-
\frac{1}{1-\alpha}Y\mathcal{V}^TE
\right)\Gamma
\end{bmatrix}.
\end{equation*}
This costs $O(nk)$ time to compute, so
\begin{math}
    Y_+=\begin{bmatrix}Y & MU\end{bmatrix}
\end{math}
can be updated in $O(nk)$ time.

Finally, we update
\begin{math}
    G_+=(I_{4(k+1)}+\mathcal{V}_+^TY_+)^{-1}.
\end{math}
After appending the new factors, we have
\begin{math}
    \mathcal{U}_+=\begin{bmatrix}\mathcal{U} & U\end{bmatrix},
    \mathcal{V}_+=\begin{bmatrix}\mathcal{V} & V\end{bmatrix},
    Y_+=M\mathcal{U}_+=\begin{bmatrix}Y & MU\end{bmatrix}.
\end{math}
Therefore,
\begin{align*}
I_{4(k+1)}+\mathcal{V}_+^TY_+
&=
\begin{bmatrix}
I_{4k}+\mathcal{V}^TY & \mathcal{V}^T MU \\
V^T Y & I_4+V^T MU
\end{bmatrix}.
\end{align*}
Since
\begin{math}
    G=(I_{4k}+\mathcal{V}^TY)^{-1}
\end{math}
is already stored, we only need to invert this block matrix. Define
\begin{equation*}
    C:=\mathcal{V}^T MU\in\R^{4k\times 4},
    \qquad
    R:=V^TY\in\R^{4\times 4k},
    \qquad
    Q:=I_4+V^T MU\in\R^{4\times 4}.
\end{equation*}
Then
\begin{equation*}
I_{4(k+1)}+\mathcal{V}_+^TY_+
=
\begin{bmatrix}
G^{-1} & C \\
R & Q
\end{bmatrix}.
\end{equation*}
Let
\begin{math}
    \Sigma:=Q-RGC\in\R^{4\times 4}
\end{math}
be the Schur complement. By the block inverse formula,
\begin{equation*}
G_+
=
\begin{bmatrix}
G+GC\Sigma^{-1}RG & -GC\Sigma^{-1} \\
-\Sigma^{-1}RG & \Sigma^{-1}
\end{bmatrix}.
\end{equation*}
It remains to bound the cost of forming this expression. The matrix $MU$ has already been computed. Forming
\begin{math}
    C=\mathcal{V}^TMU
\end{math}
costs $O(nk)$ time, forming
\begin{math}
    R=V^TY
\end{math}
costs $O(nk)$ time, and forming
\begin{math}
    Q=I_4+V^TMU
\end{math}
costs $O(n)$ time. Computing $GC$ and $RG$ costs $O(k^2)$ time, since $G$ has size $4k\times 4k$. The matrix $\Sigma$ has size $4\times 4$, so computing $\Sigma^{-1}$ costs $O(1)$ time. Hence updating $G_+$ costs
\begin{math}
    O(nk+k^2)
\end{math}
time.

Thus both the maintained state and the solution $z^*_+=B_+^{-1}\alpha s$ can be updated in $O(nk+k^2)$ time.
\end{proof}

The above lemma gives an efficient algorithm as long as $k$ is small. Thus, as $k$ grows, we need a way to propagate all changes into $B$ and make $\mathcal{U},\mathcal{V}$ small again. The following lemma gives this rebuilding step. We use the rebuilding algorithm of \citet[Theorem 2]{sankowski2004dynamic}, but state it here explicitly to avoid ambiguity. The theorem and proof of \citet{sankowski2004dynamic} are stated for the case where $O(n^\gamma)$ columns of the original matrix have changed. We note that the proof extends directly to the case where $O(n^\gamma)$ columns and $O(n^\gamma)$ rows have changed. We state this form in the following lemma.

\begin{lemma} \label{lem:sank-rebuild}
Let $A$ be nonsingular, and suppose $A'=A+\Delta$ is also nonsingular.
Assume that $\Delta$ is supported only on rows and columns indexed by a set
$S$, where $|S|=k=O(n^\gamma)$. Then $A^{-1}$ can be updated to $(A')^{-1}$ using $O\left(n^{\omega(1,\gamma,1)}\right)$ arithmetic operations.
\end{lemma}

\begin{proof}
Let $E\in\R^{n\times k}$ be the matrix whose columns are the standard basis vectors indexed by $S$. Since $\Delta$ is supported on
\begin{math}
    S\times[n]\cup[n]\times S,
\end{math}
we can write
\begin{math}
    \Delta=ER+CE^T,
\end{math}
where $R=E^T\Delta$ contains the changed rows, and
\begin{math}
    C=(I-EE^T)\Delta E
\end{math}
contains the changed columns outside those rows.

Thus
\begin{math}
    A'=A+ER+CE^T=A+UV^T,
\end{math}
where
\begin{equation*}
    U=\begin{bmatrix}E & C\end{bmatrix},
    \qquad
    V^T=\begin{bmatrix}R \\ E^T\end{bmatrix}.
\end{equation*}
Hence $U,V\in\mathbb{R}^{n\times 2k}$, so the update has rank at most $2k$. By the Woodbury identity,
\begin{equation*}
    (A')^{-1}
    =
    A^{-1}
    -
    A^{-1}U
    \left(I_{2k}+V^TA^{-1}U\right)^{-1}
    V^TA^{-1}.
\end{equation*}
The products $A^{-1}U$ and $V^TA^{-1}$ are rectangular products with one dimension $O(k)=O(n^\gamma)$, so they cost
\begin{math}
    O\left(n^{\omega(1,\gamma,1)}\right)
\end{math}
arithmetic operations. The middle matrix has size $O(k)\times O(k)$, and the final product has dimensions
\begin{math}
    n\times O(k)
\end{math}
by
\begin{math}
    O(k)\times n.
\end{math}
These steps are within the same bound. Therefore the full update costs
\begin{math}
    O\left(n^{\omega(1,\gamma,1)}\right)
\end{math}
arithmetic operations.
\end{proof}

Using the above lemma, we rebuild after $n^\gamma$ updates.

\begin{corollary} \label{cor:reset-cost}
After $n^\gamma$ edge updates, the inverse $B^{-1}$ can be
rebuilt in $O\left(n^{\omega(1,\gamma,1)}\right)$ arithmetic operations.
\end{corollary}
\begin{proof}
Each edge update changes at most two rows and two columns. Therefore, after $n^\gamma$ edge updates, the total change is supported on $O(n^\gamma)$ rows and $O(n^\gamma)$ columns. By \Cref{lem:sank-rebuild}, we can compute the updated inverse $B^{-1}$ in $O\left(n^{\omega(1,\gamma,1)}\right)$ arithmetic operations.
\end{proof}

Combining the above results gives the following theorem.
\begin{theorem}
    The fixed point $z^*$ can be maintained by an algorithm with $O(n^\omega)$ initialization cost and $O\left(n^{\omega(1,\gamma,1)-\gamma} + n^{1+\gamma}\right)$ worst-case update time.
    The algorithm maintains $z^*$ explicitly after each update. In particular, setting $\gamma \approx 0.5275$ gives update time $O\left(n^{1.5275}\right)$ under the current known bounds.
\end{theorem}
\begin{proof}
    We begin by computing $M = B^{-1}$ in $O(n^\omega)$ time. We then process updates using the lazy representation and rebuild after $O(n^\gamma)$ updates. By \Cref{lem:exact:update}, each lazy update costs $O\left(n^{1+\gamma}\right)$ time. By \Cref{cor:reset-cost}, each rebuild costs $O(n^{\omega(1,\gamma,1)})$ time. This gives the stated bound in amortized time.

    To obtain the same bound in worst-case time, we use the standard deamortization technique noted by \citet{sankowski2004dynamic}. We maintain two copies of the data structure. The active copy answers queries and processes updates using the lazy Woodbury representation. In parallel, the second copy performs the rebuild incrementally over the next $n^\gamma$ updates. By spreading the $O\left(n^{\omega(1,\gamma,1)}\right)$ rebuild work evenly over these updates, the rebuild contributes $O\left(n^{\omega(1,\gamma,1) - \gamma}\right)$ worst-case time per update. At the end of the phase, the rebuilt copy becomes active.
    Using the matrix multiplication bounds of \citet{alman2025more} and the interpolation method for $\omega(1,\gamma,1)$ from \citet{van2019dynamic}, we optimize over $\gamma$ and obtain $\gamma \approx 0.5275$, which gives the claimed bound.
\end{proof}

\section{Experiment Details of Section 4.1} \label{appendix:exp-acc}
This section describes the experimental setup for \Cref{table:exp-acc:acc}. We provide the general setup, the hyperparameter search and the configurations selected for each cell of \Cref{table:exp-acc:acc} in \Cref{appendix:exp-acc:hyper}, and a discussion of the \textsc{Chameleon} and \textsc{Squirrel} datasets in \Cref{appendix:exp-acc:cham-squ}.

\paragraph{Datasets.} \label{appendix:exp-acc:datasets}
We consider several common datasets. All datasets contain feature vectors for all nodes, and labels for a subset of nodes are provided in the training set. \textsc{Cora}, \textsc{Citeseer} \citep{sen2008collective}, and \textsc{PubMed} \citep{namata2012query} are citation graphs, where nodes are papers and edges are citations. \textsc{Actor} \citep{tang2009social,pei2020geom} is a graph of actors, where edges denote co-occurrence on the same Wikipedia page. WebKB \citep{craven1998learning} is a dataset collected from university websites where nodes represent webpages, and edges represent links between them. We use the \textsc{Cornell}, \textsc{Wisconsin}, and \textsc{Texas} graphs from this dataset. \textsc{Chameleon} and \textsc{Squirrel} are graphs of Wikipedia pages on specific topics \citep{rozemberczki2021multi,pei2020geom}, where nodes are pages and edges are mutual links between them. For a discussion of \textsc{Chameleon} and \textsc{Squirrel}, see \Cref{appendix:exp-acc:cham-squ}.

\paragraph{Splits.} All datasets are loaded through PyTorch Geometric \citep{fey2019fast}. \textsc{Cora}, \textsc{Citeseer}, and \textsc{Pubmed} use the standard Planetoid train/validation/test split \citep{yang2016revisiting}. \textsc{Cornell}, \textsc{Texas}, \textsc{Wisconsin} (WebKB), \textsc{Actor}, \textsc{Chameleon}, and \textsc{Squirrel} use the ten random $48\%/32\%/20\%$ splits released by \citet{pei2020geom}. We make every graph undirected and standardize each feature column to have zero mean and unit variance. \Cref{table:exp-acc:datasets} shows the scales of the datasets used in these experiments.

The PyTorch Geometric Planetoid distributions of \textsc{Cora}, \textsc{Citeseer}, and \textsc{Pubmed} are MIT licensed, with the underlying \textsc{Pubmed} records also subject to the NLM Data License. The WebKB graphs and \textsc{Actor} are loaded through PyTorch Geometric, but their standalone dataset licenses appear unspecified. \textsc{Chameleon} and \textsc{Squirrel}, including the filtered versions released by \citet{platonov2023critical}, are distributed under the MIT license.

\paragraph{Method.} Following \citet{zheng2022instant}, we first propagate the features using the PageRank model and then train an MLP on the propagated features using the training labels. For the non-linear model, we use the same activation function for all node functions $f_i$. In addition to the propagated features $z$, we pass the pre-activation vectors $y = \alpha s + (1 - \alpha)Wz$ to the MLP. When the activation function is the identity, $y \approx z$.

\paragraph{Baselines.}
We ran every baseline in \Cref{table:exp-acc:acc} ourselves, on exactly the same splits as our model: GCN \citep{kipf2017semi}, GAT \citep{velivckovic2017graph}, Geom-GCN \citep{pei2020geom}, APPNP \citep{gasteiger2018predict}, ACM-GCN \citep{luan2022revisiting}, and GREAD \citep{choi2023gread}. When the authors of a method published a configuration for a dataset and split, we used it. Otherwise, we tuned the method over at most 288 configurations, the same size our model uses for each activation function. The tuning has two stages. We first rank the configurations by validation accuracy on one seed. We then rerun the best configurations on five new seeds, select the one with the highest validation accuracy, and report its test accuracy from these runs. We report standard errors in the same way as for our model.

\begin{table}[htbp]
\centering
\caption{Statistics for datasets used in \Cref{section:exp-acc}.}
\label{table:exp-acc:datasets}
\input{tables/exp_acc_datasets}
\end{table}

\paragraph{Propagation.}
After normalization, we multiply the features by a feature scale $c_{\mathrm{feat}}$ before propagation. The classifier receives the concatenation $[y, z]$, where $y = \alpha s + (1-\alpha)Wz$ is the pre-activation vector. We include $c_{\mathrm{feat}}$ in the search grid because it interacts with the saturation thresholds of activations such as tanh.

\paragraph{Activations.} We evaluate the linear baseline (denoted \enquote{Linear}), $\tanh$, $\mathrm{sigmoid}$, scaled tanh with $c \in \{0.5, 1, 1.5, 2, 3, 4, 5, 6, 8\}$, hard tanh with $c \in \{0.3, 0.5, 0.7, 1, 1.5, 2, 2.5\}$, and shifted tanh with $c \in \{0.5, -0.5, -0.9, -1.2\}$. This broad set lets us test activations with different shapes and saturation behavior. \Cref{table:exp-acc:acc} reports only the activations that achieve the best test accuracy in at least one dataset column; the remaining activations were also evaluated.

\paragraph{Classifier.} The classifier is either a linear layer or a one-hidden-layer MLP with hidden width $512$, ReLU, and dropout $0.5$. The classifier type is selected as part of the search grid. Both classifiers are trained with Adam on the cross-entropy loss, with weight decay $5\times 10^{-4}$, for at most $400$ epochs. We use early stopping after $100$ epochs without improvement in validation accuracy. We report the test accuracy at the epoch with the highest validation accuracy.

\paragraph{Compute.} Each run for an (activation, dataset, seed, split, hyperparameter) configuration uses one NVIDIA RTX A4000 GPU, 4GB RAM, and 2 CPU cores and takes less than 5 seconds to complete.

\subsection{Hyperparameter Sweep and Selected Configurations} \label{appendix:exp-acc:hyper}

\begin{table}[htbp]
\centering
\caption{Hyperparameter search grid used for each activation--dataset pair.}
\label{table:hyperparameter-grid}
\input{tables/exp_acc_grid}
\end{table}

We use the same search grid in \Cref{table:hyperparameter-grid} for every (activation, dataset) pair. For each (activation, dataset) cell, we evaluate every grid point on $5$ random seeds for each of the $10$ splits described above, for a total of $50$ runs per configuration. We then select the grid point with the highest mean validation accuracy across these runs. In \Cref{table:exp-acc:acc}, we report the mean and standard error of the test accuracy for the selected grid point. The selected configurations are listed in \Cref{table:exp-acc:hyper}.

\begin{table}[htbp]
\caption{Configuration of $(\alpha, \beta, c_{\mathrm{feat}}, \text{classifier}, \text{lr})$ selected by validation accuracy for each (activation, dataset) cell of \Cref{table:exp-acc:acc}. \enquote{L} denotes the linear classifier and \enquote{M} denotes the MLP.}
\centering
\label{table:exp-acc:hyper}
\setlength{\tabcolsep}{4pt}
\begin{adjustbox}{width=\linewidth}
\input{tables/exp_acc_hyper}
\end{adjustbox}
\end{table}

\subsection{\textsc{Chameleon} and \textsc{Squirrel}} \label{appendix:exp-acc:cham-squ}

\textsc{Chameleon} and \textsc{Squirrel} \citep{rozemberczki2021multi} are Wikipedia link networks that are widely used as heterophilous benchmarks. \citet{platonov2023critical} show that the publicly distributed versions of these datasets contain duplicated nodes, which can leak across the train, validation, and test splits. This leakage can increase the accuracy of methods that effectively memorize duplicated nodes. They release filtered versions with the duplicated nodes removed. The results in \Cref{table:exp-acc:acc} use the \emph{original} Geom-GCN preprocessed splits, which are standard for these datasets. \Cref{table:exp-acc:filtered} reports our model and the baselines on the filtered versions, using the ten splits of \citet{platonov2023critical}. For our model, we use the same hyperparameter search grid as in \Cref{appendix:exp-acc:hyper}, and \Cref{table:exp-acc:filtered-hyper} shows the selected hyperparameters. The baselines are tuned as described in \Cref{appendix:exp-acc}. We do not include Geom-GCN, because its authors released the data it needs only for the original graphs.

\begin{table}[H]
\caption{Test accuracy on the filtered \textsc{Chameleon} and \textsc{Squirrel} datasets of \citet{platonov2023critical}. The best method is \underline{underlined}.}
\centering
\label{table:exp-acc:filtered}
\setlength{\tabcolsep}{4pt}
\input{tables/exp_acc_cham_squir}
\end{table}

\begin{table}[H]
\caption{Configuration of $(\alpha, \beta, c_{\mathrm{feat}}, \text{classifier}, \text{lr})$ selected by validation accuracy for each (activation, dataset) cell of \Cref{table:exp-acc:filtered}. \enquote{L} denotes the linear classifier and \enquote{M} denotes the MLP.}
\centering
\label{table:exp-acc:filtered-hyper}
\setlength{\tabcolsep}{4pt}
\input{tables/exp_acc_cham_squir_hyper}
\end{table}

\section{Experiment Details of Section 4.2} \label{appendix:exp-dyn}

\paragraph{Datasets and snapshots.}
We use two graphs from the Open Graph Benchmark \citep{hu2020open}: \texttt{arxiv}, a citation network, and \texttt{products}, the Amazon products co-purchasing network. We also use the synthetic \texttt{SBM-500K} dataset generated by \citet{zheng2022instant}. The edge updates for \texttt{arxiv} and \texttt{products} are incremental edge additions, while \texttt{SBM-500K} includes both edge additions and edge deletions.
The \texttt{arxiv} and \texttt{products} benchmarks are used with their official train/validation/test splits. We follow the snapshot split protocol of \citet{zheng2022instant}. First, we make the graph undirected and standardize every feature column to have zero mean and unit variance. We then remove every non-self edge incident to a training node. The remaining edges form the initial graph, and the removed edges form the insertion stream. The insertion stream is shuffled with constant random seed and split into contiguous snapshots. We use $16$ snapshots for \texttt{arxiv}, $15$ snapshots for \texttt{products}, and \texttt{SBM-500K} contains 10 snapshots. Self-loops are present throughout the experiment and are not part of the update stream. 

For \texttt{SBM-500K}, the dataset is generated with the InstantGNN \citep{zheng2022instant} SBM generator\textsuperscript{\ref{fn:instant-gnn}} using $500{,}000$ nodes, $50$ communities, average in-community degree $20$, average out-community degree $1$, $10$ snapshots, and $2{,}500$ changed nodes per snapshot. The dataset provides initial labels and per-snapshot labels. Since features and splits are not shipped with the text edge stream, we generate fixed-seed sparse random features with $256$ dimensions and $16$ nonzero entries per node, standardize every feature column, and use random train/validation/test splits of $70\%/20\%/10\%$.

Under this construction, \texttt{arxiv} has $169{,}343$ nodes, $128$ features, $213{,}848$ initial undirected non-self edges, and $1{,}157{,}799$ total undirected non-self edges after all insertions. The \texttt{products} graph has $2{,}449{,}029$ nodes, $100$ features, $33{,}592{,}708$ initial undirected non-self edges, and $61{,}859{,}012$ total undirected non-self edges after all insertions. The \texttt{SBM-500K} graph has $500{,}000$ nodes, $256$ generated features, $50$ classes, $6{,}820{,}831$ initial undirected non-self edges, and $6{,}823{,}982$ final undirected non-self edges after all updates. Across its $10$ snapshots, the stream contains $1{,}263{,}739$ unique undirected edge-update events, consisting of $633{,}445$ insertions and $630{,}294$ deletions.

\paragraph{Dataset licenses.}
The Open Graph Benchmark lists \texttt{arxiv} under ODC-BY and \texttt{products} under the Amazon License. The \texttt{SBM} datasets are synthetic datasets generated using the code
\footnote{\label{fn:instant-gnn}\url{https://github.com/zheng-yp/InstantGNN}}
from InstantGNN \citep{zheng2022instant}; the repository does not specify a license for these datasets.

\paragraph{Propagation methods.}
For both the linear and non-linear models, we use
\begin{math}
    \beta=0.5,\ \epsilon=10^{-7}.
\end{math}
We use $\alpha=0.1$ for \texttt{arxiv} and \texttt{products}, and $\alpha=0.001$ for \texttt{SBM-500K}.

Before propagation, we multiply the standardized input features by $3$. The linear model uses the identity activation. The non-linear model uses hard tanh,
\begin{math}
    f(x)=\min(2.5,\max(-2.5,x)).
\end{math}
We use the same activation for every node and every feature dimension.

We compare three propagation strategies. The \enquote{scratch} method applies all edge updates in the current snapshot, resets the propagation state to the canonical initialization, and runs cleanup from scratch. The \enquote{batched} method applies all insertions in the current snapshot as one batch and then runs cleanup once. The \enquote{single} method applies insertions one edge at a time, runs the dynamic update after each edge, and returns the maintained representation after every edge insertion; for reporting, we evaluate the representation after the last edge of each snapshot. Thus, \enquote{single} solves the strictly stronger online problem, while \enquote{batched} matches the snapshot update used in \citet{zheng2022instant}.

Propagation is implemented in C++. The implementation stores the feature, propagated, pre-activation, and residual arrays as dense double-precision arrays of shape $F \times n$. Cleanup is run independently for each feature dimension, and dimensions are parallelized with OpenMP. We use the same approximation guarantees as \citet{zheng2022instant}.

\paragraph{Classifier and parameters.}
After each initial propagation or snapshot update, we train a classifier from scratch on the current propagated features. The classifier input is the concatenation $[z,y]$, where
\begin{math}
    y = \alpha s + (1-\alpha)Wz
\end{math}
is the pre-activation vector. For \texttt{arxiv} and \texttt{products}, the classifier is a $4$-layer MLP with hidden width $1024$, batch normalization after hidden linear layers, ReLU activations, and Adam optimization with learning rate $10^{-4}$ and no weight decay. For \texttt{SBM-500K}, we use a $2$-layer MLP with hidden width $1024$, dropout $0.1$, Adam learning rate $0.01$, and no weight decay.
We train \texttt{arxiv} and \texttt{products} for at most $1000$ epochs and \texttt{SBM-500K} for at most $200$ epochs, with early stopping after $50$ epochs without validation improvement.
We use dropout $0.3$ on \texttt{arxiv}, dropout $0.5$ on \texttt{products}, and dropout $0.1$ on \texttt{SBM-500K}. The \texttt{arxiv} classifier is trained full-batch, while \texttt{products} uses mini-batches of size $512$ and \texttt{SBM-500K} uses mini-batches of size $1024$.

These parameters were all inherited from experiments of \citet{zheng2022instant}.

\paragraph{Compute.}
Each dynamic experiment on \texttt{arxiv}, \texttt{products}, and \texttt{SBM-500K} is run with one NVIDIA RTX A4000 GPU, $16$ CPU cores, and $32$GB RAM. The runs for \Cref{table:exp-dyn:large} use $16$ CPU cores and no GPU. They need about $9$GB of RAM on \texttt{SBM-500K}, $32$GB on \texttt{SBM-2M}, $77$GB on \texttt{SBM-5M}, and $700$GB on \texttt{papers100M}.

\paragraph{Reported quantities.}
The accuracy plots report the test accuracy of the classifier trained on the propagated features available at each snapshot. The cumulative-time plots report the cumulative propagation time, including the initial propagation and all processed snapshots. The per-snapshot time plots report only propagation time for the corresponding initial graph or snapshot update; classifier training time is excluded from these timing plots.

\paragraph{Larger graphs.}
We also run two larger versions of the SBM benchmark, \texttt{SBM-2M} and \texttt{SBM-5M}, and the \texttt{papers100M} graph from Open Graph Benchmark \citep{hu2020open}, which has 111M nodes and 1.6B edges. \Cref{table:exp-dyn:large} reports the propagation time per snapshot. The cost of recomputation grows with the size of the graph, while the cost of dynamic updates does not. We generate \texttt{SBM-2M} and \texttt{SBM-5M} with the same InstantGNN generator as \texttt{SBM-500K}, changing only the number of nodes. We scale the number of communities so that each community has $10{,}000$ nodes, and keep the degrees, the $10$ snapshots, and the $2{,}500$ changed nodes per snapshot. So each snapshot has about $126{,}000$ edge updates at every size. We use the same $\alpha$, $\beta$, and $\epsilon$ as for \texttt{SBM-500K}. For \texttt{papers100M}, we follow the same snapshot split protocol as for \texttt{arxiv} and \texttt{products}, with $20$ snapshots of about $2$M edge insertions each. We use $\alpha=0.1$, $\beta=0.5$, and $\epsilon = 2.2\times 10^{-9}$. We run each SBM setting with $5$ random seeds and \texttt{papers100M} once.

\begin{table}[htbp]
\caption{Propagation time per snapshot, in seconds, on larger graphs.}
\label{table:exp-dyn:large}
\centering
\begin{adjustbox}{width=0.7\linewidth}
\begin{tabular}{lrrrrrrr}
\toprule
& & \multicolumn{3}{c}{Hard tanh} & \multicolumn{3}{c}{Linear} \\
\cmidrule(lr){3-5} \cmidrule(lr){6-8}
Dataset & Nodes & scratch & single & batched & scratch & single & batched \\
\midrule
\texttt{SBM-500K} & 0.5M & 34.9 & 19.0 & 4.3 & 27.2 & 20.0 & 4.3 \\
\texttt{SBM-2M} & 2M & 68.3 & 7.3 & 1.3 & 55.2 & 4.6 & 1.3 \\
\texttt{SBM-5M} & 5M & 119.1 & 3.8 & 1.3 & 89.2 & 3.8 & 1.0 \\
\texttt{papers100M} & 111M & 20{,}248 & 579.7 & 116.6 & 31{,}695 & 777.9 & 151.3 \\
\bottomrule
\end{tabular}
\end{adjustbox}
\end{table}

\paragraph{Varying $\epsilon$ and $\alpha$.}
To study how the parameters affect running time, we repeat the same experiment while varying $\epsilon$ and $\alpha$ on the \texttt{arxiv} and \texttt{SBM-500K} datasets. \Cref{fig:vary-eps,fig:vary-alpha} show the results. We find that the observed running times are consistent with the theoretical dependence on these parameters. For \Cref{fig:vary-eps,fig:vary-alpha}, we vary $\epsilon \in \{10^{-7}, 10^{-6}, 10^{-5}, 10^{-4}, 10^{-3}\}$ and $\alpha \in \{0.001, 0.01, 0.1, 0.2, 0.3, 0.5, 0.7\}$, respectively. We run each setting with two random seeds and report the mean over the two runs.

\begin{figure*}[htbp]
\centering
\begin{subfigure}[t]{0.5\textwidth}
\includegraphics[width=\linewidth]{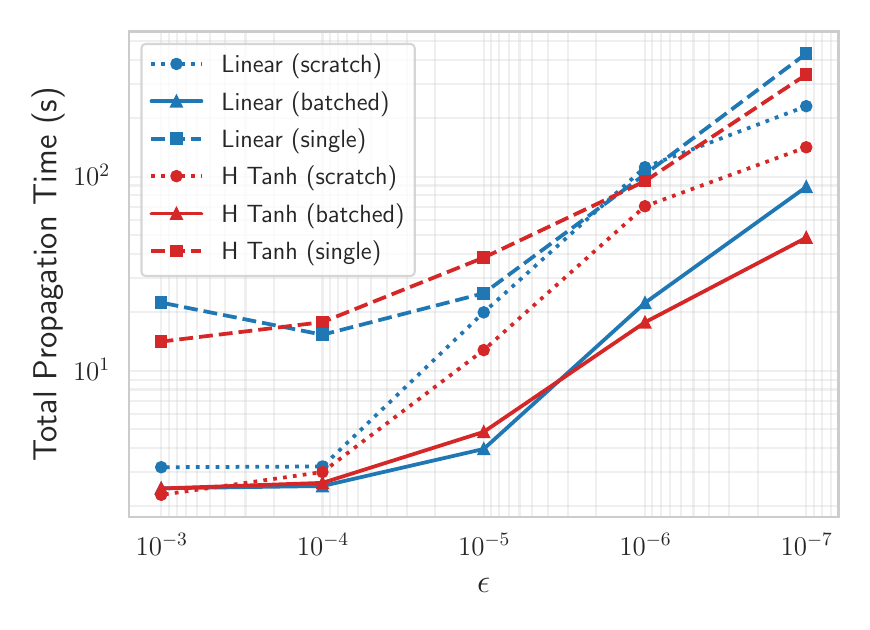}
\caption{\texttt{arxiv}}
\end{subfigure}\hfill
\begin{subfigure}[t]{0.5\textwidth}
\includegraphics[width=\linewidth]{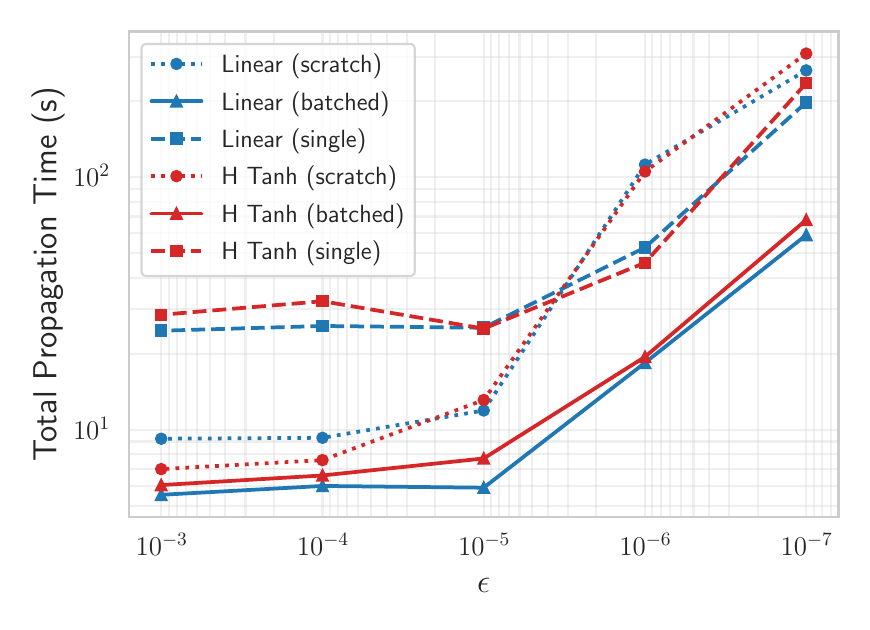}
\caption{\texttt{SBM-500K}}
\end{subfigure}\hfill
\caption{Total propagation time of varying $\epsilon$ values. The propagation time scales inversely with $\epsilon$.}
\label{fig:vary-eps}
\end{figure*}

\begin{figure*}[htbp]
\centering
\begin{subfigure}[t]{0.5\textwidth}
\includegraphics[width=\linewidth]{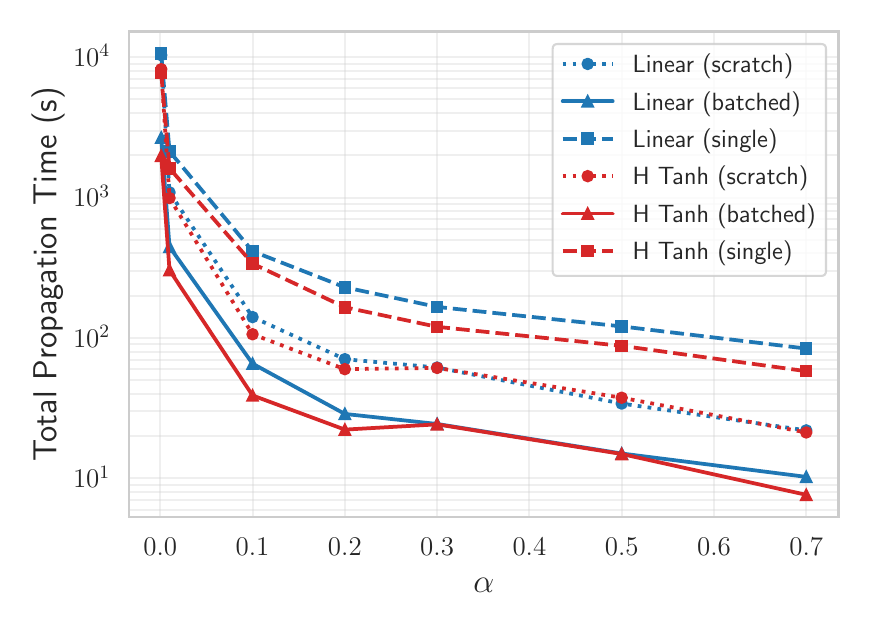}
\caption{\texttt{arxiv}}
\end{subfigure}\hfill
\begin{subfigure}[t]{0.5\textwidth}
\includegraphics[width=\linewidth]{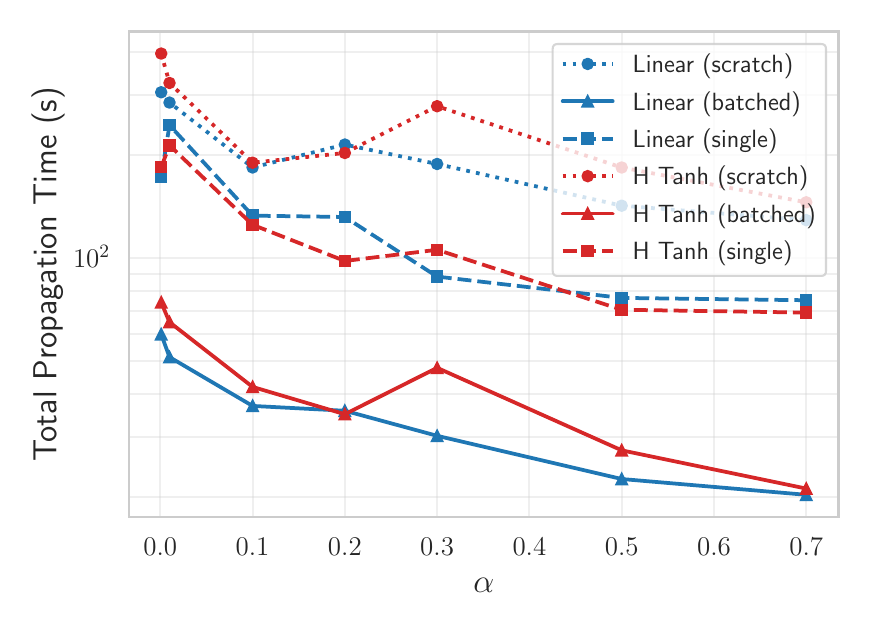}
\caption{\texttt{SBM-500K}}
\end{subfigure}\hfill
\caption{Total propagation time of varying $\alpha$ values. Our theoretical bounds show that the propagation time scales with $1/\alpha^2$.}
\label{fig:vary-alpha}
\end{figure*}

%% file: tables/exp_acc_datasets.tex
\begin{tabular}{lrrrr}
\toprule
Dataset & Nodes ($n$) & Edges & Feature dimension & Classes \\
\midrule
\textsc{Cornell} & 183 & 164 & 1,703 & 5 \\
\textsc{Texas} & 183 & 167 & 1,703 & 5 \\
\textsc{Wisconsin} & 251 & 231 & 1,703 & 5 \\
\textsc{Squirrel} & 5,201 & 159,248 & 2,089 & 5 \\
\textsc{Actor} & 7,600 & 14,992 & 932 & 5 \\
\textsc{Chameleon} & 2,277 & 23,370 & 2,325 & 5 \\
\textsc{Cora} & 2,708 & 5,278 & 1,433 & 7 \\
\textsc{CiteSeer} & 3,327 & 4,552 & 3,703 & 6 \\
\textsc{PubMed} & 19,717 & 44,324 & 500 & 3 \\
\midrule
\textsc{Chameleon (filtered)} & 890 & 8,854 & 2,325 & 5 \\
\textsc{Squirrel (filtered)} & 2,223 & 46,998 & 2,089 & 5 \\
\bottomrule
\end{tabular}

%% file: tables/exp_acc_grid.tex
\begin{tabular}{ll}
\toprule
Hyperparameter & Values \\
\midrule
$\alpha$ & $\{0.01, 0.1, 0.2, 0.3, 0.5, 0.7\}$ \\
$\beta$ & $\{0.5, 1.0\}$ \\
$c_{\mathrm{feat}}$ & $\{1, 3, 5\}$ \\
Classifier & $\{\text{linear}, \text{MLP}\}$ \\
Learning rate & $\{10^{-3}, 10^{-2}, 10^{-1}, 0.5\}$ \\
Weight decay & $5 \times 10^{-4}$ \\
MLP hidden dimension & $512$ \\
MLP dropout & $0.5$ \\
Training epochs & $400$ \\
Early stopping patience & $100$ \\
Optimizer & Adam \\
Loss & Cross-entropy \\
\bottomrule
\end{tabular}

%% file: tables/exp_acc_hyper.tex
\begin{tabular}{lccccccccc}
\toprule
Activation & \textsc{Corn.} & \textsc{Texas} & \textsc{Wisc.} & \textsc{Squi.} & \textsc{Actor} & \textsc{Cham.} & \textsc{Cora} & \textsc{Cite.} & \textsc{Pubm.} \\
\midrule
Linear & 0.7/1/1/L/0.5 & 0.7/1/1/L/0.5 & 0.7/0.5/1/M/0.1 & 0.01/0.5/3/M/0.001 & 0.7/0.5/1/M/0.5 & 0.1/1/5/M/0.01 & 0.01/1/3/M/0.5 & 0.1/1/1/L/0.1 & 0.01/0.5/3/M/0.1 \\
Sc Tanh 0.5 & 0.7/1/5/M/0.01 & 0.7/0.5/5/M/0.1 & 0.7/0.5/3/M/0.01 & 0.01/0.5/3/M/0.001 & 0.7/1/3/M/0.001 & 0.1/1/3/M/0.01 & 0.01/0.5/5/L/0.01 & 0.1/0.5/1/L/0.1 & 0.1/1/1/L/0.5 \\
Sc Tanh 8 & 0.7/1/1/M/0.1 & 0.5/1/1/M/0.1 & 0.2/0.5/5/M/0.01 & 0.01/0.5/3/M/0.01 & 0.5/1/5/M/0.001 & 0.01/1/1/M/0.01 & 0.01/0.5/1/M/0.01 & 0.01/0.5/3/M/0.01 & 0.01/0.5/1/M/0.01 \\
Sh Tanh -1.2 & 0.5/1/3/M/0.01 & 0.3/0.5/3/M/0.01 & 0.3/0.5/5/M/0.01 & 0.5/1/1/M/0.001 & 0.3/1/1/M/0.001 & 0.7/1/1/L/0.01 & 0.5/1/5/M/0.01 & 0.7/0.5/5/M/0.01 & 0.1/0.5/1/L/0.1 \\
H Tanh 2.5 & 0.7/1/5/M/0.01 & 0.7/1/5/M/0.01 & 0.7/0.5/3/M/0.01 & 0.01/0.5/3/M/0.001 & 0.7/1/3/M/0.001 & 0.1/1/3/M/0.01 & 0.1/1/5/M/0.5 & 0.1/0.5/1/L/0.1 & 0.1/1/3/L/0.5 \\
\bottomrule
\end{tabular}

%% file: tables/exp_acc_cham_squir.tex
\begin{tabular}{lcc}
\toprule
Method & \textsc{Cham. (filt.)} & \textsc{Squi. (filt.)} \\
\midrule
GCN & 35.6{\scriptsize\,$\pm$\,0.8} & 34.8{\scriptsize\,$\pm$\,0.4} \\
GAT & 35.3{\scriptsize\,$\pm$\,0.8} & 34.9{\scriptsize\,$\pm$\,0.4} \\
APPNP & 36.1{\scriptsize\,$\pm$\,0.9} & 34.9{\scriptsize\,$\pm$\,0.4} \\
ACM-GCN & 38.9{\scriptsize\,$\pm$\,0.7} & 35.9{\scriptsize\,$\pm$\,0.5} \\
GREAD & 44.0{\scriptsize\,$\pm$\,1.0} & 41.0{\scriptsize\,$\pm$\,0.5} \\
\midrule
Linear & 41.1{\scriptsize\,$\pm$\,1.1} & 40.4{\scriptsize\,$\pm$\,0.6} \\
Sh Tanh -0.5 & \underline{45.4{\scriptsize\,$\pm$\,1.2}} & 44.0{\scriptsize\,$\pm$\,0.5} \\
Sh Tanh -1.2 & 43.3{\scriptsize\,$\pm$\,0.9} & \underline{44.6{\scriptsize\,$\pm$\,0.6}} \\
\bottomrule
\end{tabular}

%% file: tables/exp_acc_cham_squir_hyper.tex
\begin{tabular}{lcc}
\toprule
Activation & \textsc{Cham. (filt.)} & \textsc{Squi. (filt.)} \\
\midrule
Linear & 0.01/0.5/1/M/0.1 & 0.01/0.5/1/M/0.1 \\
Sh Tanh -0.5 & 0.7/0.5/3/M/0.001 & 0.5/0.5/1/M/0.01 \\
Sh Tanh -1.2 & 0.5/0.5/5/M/0.01 & 0.7/0.5/1/M/0.001 \\
\bottomrule
\end{tabular}

%% file: ref.bib
@inproceedings{van2019dynamic,
  title={Dynamic matrix inverse: Improved algorithms and matching conditional lower bounds},
  author={Van Den Brand, Jan and Nanongkai, Danupon and Saranurak, Thatchaphol},
  booktitle={2019 IEEE 60th Annual Symposium on Foundations of Computer Science (FOCS)},
  pages={456--480},
  year={2019},
  organization={IEEE}
}

@inproceedings{sankowski2004dynamic,
  title={Dynamic transitive closure via dynamic matrix inverse},
  author={Sankowski, Piotr},
  booktitle={45th Annual IEEE Symposium on Foundations of Computer Science},
  pages={509--517},
  year={2004},
  organization={IEEE}
}

@inproceedings{zheng2022instant,
  title={Instant graph neural networks for dynamic graphs},
  author={Zheng, Yanping and Wang, Hanzhi and Wei, Zhewei and Liu, Jiajun and Wang, Sibo},
  booktitle={Proceedings of the 28th ACM SIGKDD conference on knowledge discovery and data mining},
  pages={2605--2615},
  year={2022}
}

@inproceedings{zhang2016approximate,
  title={Approximate personalized pagerank on dynamic graphs},
  author={Zhang, Hongyang and Lofgren, Peter and Goel, Ashish},
  booktitle={Proceedings of the 22nd ACM SIGKDD international conference on knowledge discovery and data mining},
  pages={1315--1324},
  year={2016}
}

@article{bunch1974triangular,
  title={Triangular factorization and inversion by fast matrix multiplication},
  author={Bunch, James R and Hopcroft, John E},
  journal={Mathematics of Computation},
  volume={28},
  number={125},
  pages={231--236},
  year={1974}
}

@book{woodbury1950inverting,
  title={Inverting modified matrices},
  author={Woodbury, Max A},
  year={1950},
  publisher={Department of Statistics, Princeton University}
}

@inproceedings{alman2025more,
  title={More asymmetry yields faster matrix multiplication},
  author={Alman, Josh and Duan, Ran and Williams, Virginia Vassilevska and Xu, Yinzhan and Xu, Zixuan and Zhou, Renfei},
  booktitle={Proceedings of the 2025 Annual ACM-SIAM Symposium on Discrete Algorithms (SODA)},
  pages={2005--2039},
  year={2025},
  organization={SIAM}
}

@article{banach1922operations,
  title={Sur les op{\'e}rations dans les ensembles abstraits et leur application aux {\'e}quations int{\'e}grales},
  author={Banach, Stefan},
  journal={Fundamenta mathematicae},
  volume={3},
  number={1},
  pages={133--181},
  year={1922},
  publisher={Polska Akademia Nauk. Instytut Matematyczny PAN}
}

@article{sen2008collective,
  title={Collective classification in network data},
  author={Sen, Prithviraj and Namata, Galileo and Bilgic, Mustafa and Getoor, Lise and Galligher, Brian and Eliassi-Rad, Tina},
  journal={AI magazine},
  volume={29},
  number={3},
  pages={93--93},
  year={2008}
}

@inproceedings{namata2012query,
  title={Query-driven active surveying for collective classification},
  author={Namata, Galileo and London, Ben and Getoor, Lise and Huang, Bert and Edu, U},
  booktitle={10th international workshop on mining and learning with graphs},
  volume={8},
  pages={1},
  year={2012}
}

@inproceedings{tang2009social,
  title={Social influence analysis in large-scale networks},
  author={Tang, Jie and Sun, Jimeng and Wang, Chi and Yang, Zi},
  booktitle={Proceedings of the 15th ACM SIGKDD international conference on Knowledge discovery and data mining},
  pages={807--816},
  year={2009}
}

@techreport{craven1998learning,
  title={Learning to extract symbolic knowledge from the world wide web},
  author={Craven, Mark and McCallum, Andrew and PiPasquo, Dan and Mitchell, Tom and Freitag, Dayne},
  year={1998}
}

@article{rozemberczki2021multi,
  title={Multi-scale attributed node embedding},
  author={Rozemberczki, Benedek and Allen, Carl and Sarkar, Rik},
  journal={Journal of Complex Networks},
  volume={9},
  number={2},
  pages={cnab014},
  year={2021},
  publisher={Oxford University Press}
}

@article{hu2020open,
  title={Open graph benchmark: Datasets for machine learning on graphs},
  author={Hu, Weihua and Fey, Matthias and Zitnik, Marinka and Dong, Yuxiao and Ren, Hongyu and Liu, Bowen and Catasta, Michele and Leskovec, Jure},
  journal={Advances in neural information processing systems},
  volume={33},
  pages={22118--22133},
  year={2020}
}

@inproceedings{pei2020geom,
  title={Geom-GCN: Geometric Graph Convolutional Networks},
  author={Pei, Hongbin and Wei, Bingzhe and Chang, Kevin Chen-Chuan and Lei, Yu and Yang, Bo},
  booktitle={International Conference on Learning Representations},
  year={2020}
}

@article{gasteiger2018predict,
  title={Predict then propagate: Graph neural networks meet personalized pagerank},
  author={Gasteiger, Johannes and Bojchevski, Aleksandar and G{\"u}nnemann, Stephan},
  journal={arXiv preprint arXiv:1810.05997},
  year={2018}
}

@article{fey2019fast,
  title={Fast graph representation learning with PyTorch Geometric},
  author={Fey, Matthias and Lenssen, Jan Eric},
  journal={arXiv preprint arXiv:1903.02428},
  year={2019}
}

@inproceedings{yang2016revisiting,
  title={Revisiting semi-supervised learning with graph embeddings},
  author={Yang, Zhilin and Cohen, William and Salakhudinov, Ruslan},
  booktitle={International conference on machine learning},
  pages={40--48},
  year={2016},
  organization={PMLR}
}

@article{platonov2023critical,
  title={A critical look at the evaluation of GNNs under heterophily: Are we really making progress?},
  author={Platonov, Oleg and Kuznedelev, Denis and Diskin, Michael and Babenko, Artem and Prokhorenkova, Liudmila},
  journal={arXiv preprint arXiv:2302.11640},
  year={2023}
}

@article{wu2020comprehensive,
  title={A comprehensive survey on graph neural networks},
  author={Wu, Zonghan and Pan, Shirui and Chen, Fengwen and Long, Guodong and Zhang, Chengqi and Yu, Philip S},
  journal={IEEE transactions on neural networks and learning systems},
  volume={32},
  number={1},
  pages={4--24},
  year={2020},
  publisher={IEEE}
}

@inproceedings{kipf2017semi,
  title={Semi-Supervised Classification with Graph Convolutional Networks},
  author={Kipf, Thomas N and Welling, Max},
  booktitle={International Conference on Learning Representations},
  year={2017}
}

@article{hamilton2017inductive,
  title={Inductive representation learning on large graphs},
  author={Hamilton, Will and Ying, Zhitao and Leskovec, Jure},
  journal={Advances in neural information processing systems},
  volume={30},
  year={2017}
}

@article{BahmaniCG10,
  author       = {Bahman Bahmani and
                  Abdur Chowdhury and
                  Ashish Goel},
  title        = {Fast Incremental and Personalized PageRank},
  journal      = {Proc. {VLDB} Endow.},
  volume       = {4},
  number       = {3},
  pages        = {173--184},
  year         = {2010},
  url          = {http://www.vldb.org/pvldb/vol4/p173-bahmani.pdf},
  doi          = {10.14778/1929861.1929864},
  bibsource    = {dblp computer science bibliography, https://dblp.org}
}

@article{zheng2023decoupled,
  title={Decoupled graph neural networks for large dynamic graphs},
  author={Zheng, Yanping and Wei, Zhewei and Liu, Jiajun},
  journal={arXiv preprint arXiv:2305.08273},
  year={2023}
}

@article{rossi2020temporal,
  title={Temporal graph networks for deep learning on dynamic graphs},
  author={Rossi, Emanuele and Chamberlain, Ben and Frasca, Fabrizio and Eynard, Davide and Monti, Federico and Bronstein, Michael},
  journal={arXiv preprint arXiv:2006.10637},
  year={2020}
}

@inproceedings{trivedi2019dyrep,
  title={Dyrep: Learning representations over dynamic graphs},
  author={Trivedi, Rakshit and Farajtabar, Mehrdad and Biswal, Prasenjeet and Zha, Hongyuan},
  booktitle={International conference on learning representations},
  year={2019}
}

@article{kumar2018learning,
  title={Learning dynamic embeddings from temporal interactions},
  author={Kumar, Srijan and Zhang, Xikun and Leskovec, Jure},
  journal={arXiv preprint arXiv:1812.02289},
  year={2018}
}

@inproceedings{
Xu2020Inductive,
title={Inductive representation learning on temporal graphs},
author={da Xu and chuanwei ruan and evren korpeoglu and sushant kumar and kannan achan},
booktitle={International Conference on Learning Representations},
year={2020},
url={https://openreview.net/forum?id=rJeW1yHYwH}
}

@article{zheng2025survey,
  title={A survey of dynamic graph neural networks},
  author={Zheng, Yanping and Yi, Lu and Wei, Zhewei},
  journal={Frontiers of Computer Science},
  volume={19},
  number={6},
  pages={196323},
  year={2025},
  publisher={Springer}
}

@inproceedings{andersen2006local,
  title={Local graph partitioning using pagerank vectors},
  author={Andersen, Reid and Chung, Fan and Lang, Kevin},
  booktitle={2006 47th annual IEEE symposium on foundations of computer science (FOCS'06)},
  pages={475--486},
  year={2006},
  organization={IEEE}
}

@inproceedings{bahmani2012pagerank,
  title={Pagerank on an evolving graph},
  author={Bahmani, Bahman and Kumar, Ravi and Mahdian, Mohammad and Upfal, Eli},
  booktitle={Proceedings of the 18th ACM SIGKDD international conference on Knowledge discovery and data mining},
  pages={24--32},
  year={2012}
}

@inproceedings{JayaramLMOS24,
  author       = {Rajesh Jayaram and
                  Jakub Lacki and
                  Slobodan Mitrovic and
                  Krzysztof Onak and
                  Piotr Sankowski},
  editor       = {Karl Bringmann and
                  Martin Grohe and
                  Gabriele Puppis and
                  Ola Svensson},
  title        = {Dynamic PageRank: Algorithms and Lower Bounds},
  booktitle    = {51st International Colloquium on Automata, Languages, and Programming,
                  {ICALP} 2024, Tallinn, Estonia, July 8-12, 2024},
  series       = {LIPIcs},
  pages        = {90:1--90:19},
  publisher    = {Schloss Dagstuhl - Leibniz-Zentrum f{\"{u}}r Informatik},
  year         = {2024},
  url          = {https://doi.org/10.4230/LIPIcs.ICALP.2024.90},
  doi          = {10.4230/LIPICS.ICALP.2024.90},
  bibsource    = {dblp computer science bibliography, https://dblp.org}
}

@article{velivckovic2017graph,
  title={Graph attention networks},
  author={Veli{\v{c}}kovi{\'c}, Petar and Cucurull, Guillem and Casanova, Arantxa and Romero, Adriana and Lio, Pietro and Bengio, Yoshua},
  journal={arXiv preprint arXiv:1710.10903},
  year={2017}
}

@article{scarselli2008graph,
  title={The graph neural network model},
  author={Scarselli, Franco and Gori, Marco and Tsoi, Ah Chung and Hagenbuchner, Markus and Monfardini, Gabriele},
  journal={IEEE transactions on neural networks},
  volume={20},
  number={1},
  pages={61--80},
  year={2008},
  publisher={IEEE}
}

@article{luan2022revisiting,
  title={Revisiting heterophily for graph neural networks},
  author={Luan, Sitao and Hua, Chenqing and Lu, Qincheng and Zhu, Jiaqi and Zhao, Mingde and Zhang, Shuyuan and Chang, Xiao-Wen and Precup, Doina},
  journal={Advances in neural information processing systems},
  volume={35},
  pages={1362--1375},
  year={2022}
}

@inproceedings{choi2023gread,
  title={{GREAD}: Graph Neural Reaction-Diffusion Networks},
  author={Choi, Jeongwhan and Hong, Seoyoung and Park, Noseong and Cho, Sung-Bae},
  booktitle={Proceedings of the 40th International Conference on Machine Learning},
  series={Proceedings of Machine Learning Research},
  volume={202},
  year={2023}
}
